\pdfoutput=1
\documentclass[runningheads]{llncs}

\usepackage[year=2026]{eccv}

\usepackage{eccvabbrv}
\usepackage{graphicx}
\usepackage{booktabs}
\usepackage{amsmath}
\usepackage{amssymb}
\usepackage{adjustbox}
\usepackage[accsupp]{axessibility}
\usepackage{multirow}
\usepackage{xcolor}

\usepackage{hyperref}

\newcommand{\tabref}[1]{Table~\ref{#1}}

\newcommand{\code}[1]{\nolinkurl{#1}}

\newcommand{\SingleDomainExtremeShiftFpr}{50\%} 
\newcommand{\DSCRhoKLow}{0.80} 
\newcommand{\DSCRhoKHigh}{0.95} 
\newcommand{\DSCTopPCCount}{64} 
\newcommand{\DSCTopPCVarCapture}{88\%} 
\newcommand{\DSCReffIncreaseSplits}{39/40} 
\newcommand{\DSCColonReffFrom}{5.2} 
\newcommand{\DSCColonReffTo}{12.1} 
\newcommand{\DSCYogaReffFrom}{4.5} 
\newcommand{\DSCYogaReffTo}{47.0} 
\newcommand{\DSCEuroSATReffFrom}{8.1} 
\newcommand{\DSCEuroSATReffTo}{15.9} 

\newcommand{\NearoodViMTGTDino}{59.07\%} 
\newcommand{\NearoodViMStdDino}{64.10\%} 
\newcommand{\NearoodViMImprovementPP}{5.03} 
\newcommand{\NearoodMDSTGTDino}{64.17\%} 
\newcommand{\NearoodMDSStdDino}{68.30\%} 
\newcommand{\NearoodMDSStdResNet}{70.12\%} 
\newcommand{\NearoodMDSTGTResNet}{59.07\%} 
\newcommand{\NearoodMDSImprovementPPResNet}{11.05} 
\newcommand{\NearoodViMStdResNet}{70.12\%} 
\newcommand{\NearoodViMTGTResNet}{60.68\%} 
\newcommand{\NearoodViMImprovementPPResNet}{9.44} 
\newcommand{\NearoodMDSImprovementPPDino}{4.13} 

\newcommand{\TeacherOnlyFaroodFPR}{1.09\%} 
\newcommand{\TeacherOnlyNearoodFPR}{81.21\%} 

\newcommand{\FaroodMDSStdDino}{22.30\%} 
\newcommand{\FaroodMDSTGTDino}{21.46\%} 
\newcommand{\FaroodMDSStdResNet}{25.55\%} 
\newcommand{\FaroodMDSTGTResNet}{13.94\%} 
\newcommand{\FaroodMDSImprovementPPResNet}{11.61} 
\newcommand{\FaroodViMStdResNet}{26.43\%} 
\newcommand{\FaroodViMTGTResNet}{15.65\%} 
\newcommand{\FaroodViMImprovementPPResNet}{10.78} 
\newcommand{\FaroodKnnStdResNet}{34.98\%} 
\newcommand{\FaroodKnnTGTResNet}{20.06\%} 
\newcommand{\FaroodKnnImprovementPPResNet}{12.87} 
\newcommand{\FaroodEBOStdResNet}{52.69\%} 
\newcommand{\FaroodEBOTGTResNet}{41.73\%} 
\newcommand{\FaroodEBOImprovementPPResNet}{8.93} 
\newcommand{\FaroodKnnStdDino}{28.69\%} 
\newcommand{\FaroodKnnTGTDino}{29.09\%} 

\newcommand{\AccDeltaRockPP}{+7.7} 
\newcommand{\AccDeltaFoodPP}{+4.8} 
\newcommand{\AccDeltaYogaPP}{-4.1} 
\newcommand{\AccMaxDegOtherPP}{0.01} 

\begin{document}

\title{Beyond the Class Subspace: Teacher-Guided Training for Reliable Out-of-Distribution Detection in Single-Domain Models}
\titlerunning{Beyond the Class Subspace: TGT for OOD Detection}

\author{Hong Yang\inst{1} \and
Devroop Kar\inst{1} \and
Qi Yu\inst{1} \and
Travis Desell\inst{1} \and
Alex Ororbia\inst{1}}
\authorrunning{H.~Yang et al.}

\institute{Rochester Institute of Technology, Rochester, NY 14623, USA\\
\email{hy3134@rit.edu}}

\maketitle

\begin{abstract}
Out-of-distribution (OOD) detection methods perform well on multi-domain benchmarks, yet many practical systems are trained on single-domain data. We show that this regime induces a geometric failure mode, Domain-Sensitivity Collapse (DSC): supervised training compresses features into a low-rank class subspace and suppresses directions that carry domain-shift signal. We provide theory showing that, under DSC, distance- and logit-based OOD scores lose sensitivity to domain shift. We then introduce \emph{Teacher-Guided Training} (TGT), which distills class-suppressed residual structure from a frozen multi-domain teacher (DINOv2) into the student during training. The teacher and auxiliary head are discarded after training, adding no inference overhead. Across eight single-domain benchmarks, TGT yields large far-OOD FPR@95 reductions for distance-based scorers: MDS improves by \FaroodMDSImprovementPPResNet{}\,pp, ViM by \FaroodViMImprovementPPResNet{}\,pp, and kNN by \FaroodKnnImprovementPPResNet{}\,pp (ResNet-50 average), while maintaining or slightly improving in-domain OOD and classification accuracy.
\end{abstract}

\section{Introduction}
\label{sec:intro}

Out-of-distribution (OOD) detection is typically evaluated in settings where the training set itself is visually diverse, such as CIFAR-10/100 or ImageNet variants~\cite{Yang2022openood,Yang2024survey}. In practice, however, many deployed systems are trained on \emph{single-domain} data: one acquisition pipeline, one texture regime, one visual style, and a stable set of domain-shared features (e.g., histopathology, satellite imagery, industrial inspection). In this regime, a detector must identify both (i) unseen within-domain classes and (ii) samples from entirely different domains. This is precisely where many strong post-hoc OOD methods degrade sharply. In our benchmarks, even extreme shifts can yield FPR@95 above \SingleDomainExtremeShiftFpr{} for standard single-domain models.


\begin{figure*}[t]
  \centering
  \includegraphics[width=0.94\textwidth]{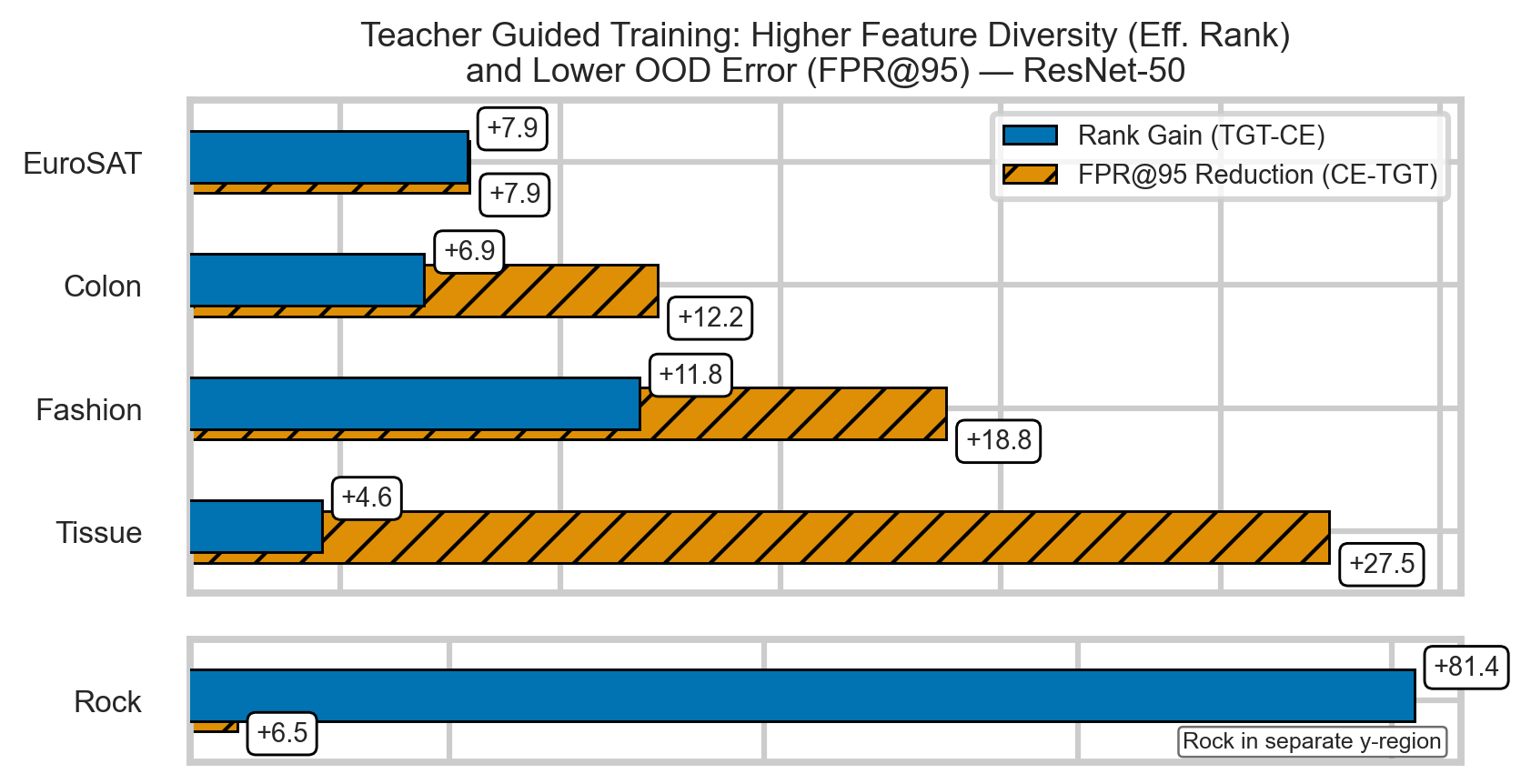}
  \caption{Per-dataset gains from Teacher-Guided Training on ResNet-50 relative to the CE baseline. Blue bars show effective-rank increase (TGT$-$CE), and hatched orange bars show FPR@95 reduction (CE$-$TGT; larger is better). The shown datasets are EuroSAT, Colon, Fashion, Tissue, and Rock (Rock shown in a separate $y$-region).}
  \label{fig:rank-fpr-gains-resnet50}
\end{figure*}

We argue that this failure is primarily geometric. Supervised single-domain training drives representations toward low-rank, class-aligned structure: variance is concentrated in a class-discriminative subspace, while orthogonal directions are suppressed for both ID and OOD samples. We formalize this phenomenon as \emph{Domain-Sensitivity Collapse} (DSC). Under DSC, both distance-based scorers (MDS, kNN, ViM) and logit-based scorers (MSP, Energy) become insensitive to domain shift because the feature dimensions that should express domain discrepancy are precisely the ones that are collapsed. \Cref{fig:rank-fpr-gains-resnet50} confirms this directly: as effective-rank gain increases (i.e.\ feature diversity is restored), FPR@95 reduction grows proportionally.

To counter this collapse, we introduce \emph{Teacher-Guided Training} (TGT), a training-time objective that transfers domain-shift sensitivity from a frozen multi-domain foundation model (DINOv2~\cite{Oquab2024dinov2}) into a single-domain student. The key signal is a \emph{class-suppressed teacher residual}: we project out class-discriminative directions and supervise the student on the remaining residual via an auxiliary domain head. This preserves standard cross-entropy classification while injecting domain-aware structure that supervised single-domain learning alone cannot recover. Importantly, the teacher and auxiliary head are discarded after training; inference uses only the student and any standard post-hoc OOD scorer.

Across eight single-domain benchmarks, TGT reduces out-of-domain FPR@95 by double-digit margins for distance-based scorers (MDS, ViM, kNN) and closes a substantial portion of the gap to a teacher-feature oracle, while improving in-domain OOD and maintaining classification accuracy.

\noindent Our contributions:
\begin{itemize}
  \item We identify and formalize \emph{Domain-Sensitivity Collapse} (DSC) as a root-cause failure mode of single-domain OOD detection, supported by theory and testable geometry predictions.
  \item We propose TGT, a training-time auxiliary loss that restores domain-sensitive geometry with no inference overhead.
  \item We validate on eight benchmarks, showing consistent FPR@95 reductions in both in- and out-of-domain settings.
\end{itemize}

\section{Related Work}
\label{sec:related}

\subsection{Out-of-Distribution Detection}
\label{sec:related:ood}

The dominant paradigm in OOD detection constructs a scoring function from a trained classifier's representations or outputs and flags below-threshold inputs as out-of-distribution.

\noindent\textbf{Logit- and softmax-based methods.}
Maximum Softmax Probability (MSP)~\cite{Hendrycks2017msp} thresholds the peak softmax value.
ODIN~\cite{Liang2018odin} sharpens the softmax via temperature scaling and input perturbation.
Energy~\cite{Liu2020energy} replaces the softmax with the log-sum-exp of logits.
ReAct~\cite{Sun2021react} and DICE~\cite{Sun2022dice} apply activation truncation and weight sparsification before computing logits, respectively.

\noindent\textbf{Distance-based methods.}
The Mahalanobis detector (MDS)~\cite{Lee2018maha} computes class-conditional Mahalanobis distances in feature space.
KNN~\cite{Sun2022knn} uses $k$-nearest-neighbor distances to training features.
ViM~\cite{Wang2022vim} augments logits with the residual norm in a principal subspace.
SSD+~\cite{Sehwag2021ssd} combines self-supervised features with Mahalanobis scoring.
Contrastive representations improve OOD separation in several related approaches~\cite{Tack2020csi,Winkens2020contrastive,Khosla2020supcon}.

All of these methods achieve strong results on multi-domain benchmarks such as CIFAR-10/100 and ImageNet~\cite{Yang2022openood,Yang2024survey}, where in-distribution training data spans rich visual diversity.
Their performance degrades in single-domain regimes (medical imaging, remote sensing, industrial inspection~\cite{Cao2020medood,Berger2021medood})---a failure mode standard benchmarks overlook because their training sets are inherently multi-domain.
Our work targets that gap.

\subsection{Feature Geometry, Neural Collapse, and Representation Rank}
\label{sec:related:nc}

\noindent\textbf{Neural collapse.}
Papyan~\etal~\cite{Papyan2020nc} showed that deep classifiers converge to an \emph{equiangular tight frame} (ETF) geometry: within-class variance vanishes, class means form the vertices of a $(C{-}1)$-dimensional simplex, and the linear head aligns with these means.
Subsequent theoretical work confirms this as a global minimizer of the cross-entropy loss under mild conditions~\cite{Zhu2021nc}.
Haas~\etal~\cite{Haas2022ncood} observed that neural-collapse-adjacent representations harm OOD separability and that $L_2$ normalization partially mitigates this.
Liu~\etal~\cite{Liu2023nci} leverage ETF geometry directly: their Neural Collapse Inspired (NCI) detector scores OOD inputs by the alignment of the mean-centred feature with the predicted class weight vector, complemented by a feature-norm term.

\noindent\textbf{Dimensional and rank collapse.}
Beyond neural collapse, contrastive and supervised representations suffer from \emph{dimensional collapse}---a reduction in the effective rank of learned features~\cite{Jing2022rank}.
Galanti~\etal~\cite{Galanti2022collapse} showed that neural collapse persists across transfer tasks and shapes fine-tuning dynamics.
In the single-domain setting we study, supervised training drives a particularly severe form: essentially all feature variance is captured by the class-discriminative subspace $\mathcal{S}_{\mathrm{cls}}$ (rank $\leq C{-}1$), and directions orthogonal to $\mathcal{S}_{\mathrm{cls}}$ are suppressed for both ID and OOD inputs alike.
We formalize this as \emph{Domain-Sensitivity Collapse} (DSC) and identify it as the direct cause of distance-based OOD failure (\cref{sec:theory}).

\subsection{Knowledge Distillation and Teacher--Student Frameworks}
\label{sec:related:kd}

Hinton~\etal~\cite{Hinton2015kd} introduced knowledge distillation (KD), training a student to match the softened output distribution of a teacher.
Subsequent work extended teacher--student transfer to vision transformers~\cite{Touvron2021deit} and self-supervised pre-training~\cite{Caron2021dino,Oquab2024dinov2}.
Yang and Xu~\cite{Yang2025pskd} use KD losses to sharpen OOD separation in the student.

\subsection{Fine-Tuning and OOD Detection}
\label{sec:related:ft}

Fort~\etal~\cite{Fort2021pretrained} found that pre-trained features provide strong OOD baselines but that fine-tuning can degrade this advantage.
Kumar~\etal~\cite{Kumar2022lp} showed that full fine-tuning can distort pretrained features and underperform linear probing by up to 7\% on out-of-distribution benchmarks, even when gaining in-distribution accuracy.
In the single-domain setting, fine-tuning a multi-domain pre-trained model reintroduces the rank collapse characteristic of DSC: effective rank declines over training while FPR@95 tends to rise, even from a rich initialization~(\cref{sec:theory}).
Preserving pretrained geometry during single-domain fine-tuning therefore requires an explicit objective beyond cross-entropy.

\section{Domain-Sensitivity Collapse}
\label{sec:theory}

This section develops a formal account of DSC.
We state sufficient conditions for quantifiable OOD-score failure (\cref{sec:theory:aniso}), identify DSC as an inherent consequence of supervised single-domain training (\cref{sec:theory:cause}), and validate predictions empirically (\cref{sec:theory:empirical}).
All results are \emph{diagnostic upper bounds}---they characterize \emph{when} and \emph{why} standard scores fail, not tight convergence guarantees.

\paragraph{Single-domain setting.}
A training set is \emph{single-domain} when all classes share a coherent visual context---a sensor, style, or semantic subcategory whose boundary is obvious to a human but invisible to the supervised loss.
Formally, there exist domain-level features that separate this context from the broader visual world, and these features are \emph{class-invariant}: they carry no signal about $y$ within the training distribution.
All eight of our benchmarks satisfy this criterion (histopathology, satellite imagery, food, rock, waste, fashion, tissue, yoga); what varies across classes is fine-grained subcategory identity, not domain.
Multi-domain benchmarks (CIFAR-10, ImageNet) do not: their classes span heterogeneous visual domains, so domain cues are entangled with class structure.

\paragraph{Mechanism and severity.}
Because domain features are class-invariant, the cross-entropy objective provides no gradient to preserve them; training therefore drives the representation toward exactly the class-discriminative subspace.
How far this compression proceeds depends on \emph{within-class visual diversity}.
When intra-class images are visually homogeneous, the network fully collapses within-class variance and the effective rank approaches the number of classes ($r_{\mathrm{eff}}/C \approx 1$, severe DSC).
When within-class diversity is intrinsically high, the same objective cannot eliminate $\Sigma_{\mathrm{within}}$ without discarding classification-relevant cues, so the representation retains higher rank ($r_{\mathrm{eff}}/C \gg 1$, mild DSC).

We quantify severity with two diagnostics:
the effective-rank ratio $r_{\mathrm{eff}}/C$ and the within-class variance fraction
$\rho_{\mathrm{within}} = \operatorname{tr}(\Sigma_{\mathrm{within}}) / \operatorname{tr}(\Sigma)$
(defined formally in \cref{sec:theory:aniso}; per-dataset values in supplementary Tables~\ref{tab:supp:rank}--\ref{tab:supp:within-class-var}).
Six of our eight benchmarks (Colon, Fashion, Food, Tissue, Yoga, EuroSAT) exhibit
$r_{\mathrm{eff}}/C \in [0.86, 1.15]$ and $\rho_{\mathrm{within}} \le 0.07$%
\footnote{Yoga is an intermediate case ($\rho_{\mathrm{within}} = 0.651$) that still collapses to $r_{\mathrm{eff}}/C = 1.13$ because it has only $C{=}4$ classes.},
confirming severe DSC\footnote{$C$ denotes training classes only; see \cref{sec:experiments} for the withheld-class protocol.}.
The remaining two---Rock ($r_{\mathrm{eff}}/C \approx 62$, $\rho_{\mathrm{within}} = 0.895$)
and Garbage ($r_{\mathrm{eff}}/C \approx 13$, $\rho_{\mathrm{within}} = 0.772$)---retain
substantial non-class structure due to high intra-class heterogeneity.
We retain them as \emph{natural controls} that confirm the theory correctly predicts
\emph{when} severe collapse occurs.

\paragraph{Connection to neural collapse.}
The equiangular tight frame (ETF) geometry of neural collapse~\cite{Papyan2020nc,Zhu2021nc} is the extreme of DSC: $\Sigma_{\mathrm{within}} \to 0$, class means form the vertices of a $(C{-}1)$-simplex, and all $d - (C{-}1)$ orthogonal dimensions have exactly zero variance.
This geometry is optimal for classification but catastrophic for OOD detection, since domain-shift directions are mapped into the null space.
DSC generalizes neural collapse by characterizing anisotropy at any training stage, not only at convergence, and by predicting which datasets will exhibit rank deficiency---making it a testable claim.

\subsection{Anisotropic Geometry and Distance Failure}
\label{sec:theory:aniso}

\paragraph{Covariance decomposition.}
Let $z = f_\theta(x) \in \mathbb{R}^d$ be the learned feature representation for a $C$-class single-domain dataset.
The empirical covariance on in-distribution (ID) data decomposes as
\begin{equation}
  \Sigma = \Sigma_{\mathrm{between}} + \Sigma_{\mathrm{within}},
  \label{eq:cov-decomp}
\end{equation}
where $\Sigma_{\mathrm{between}} = \sum_{c=1}^C p_c (\mu_c - \mu)(\mu_c - \mu)^\top$ captures inter-class spread and $\Sigma_{\mathrm{within}} = \sum_{c=1}^C p_c\,\mathbb{E}[(z-\mu_c)(z-\mu_c)^\top \mid y{=}c]$ captures intra-class scatter.
Because $\Sigma_{\mathrm{between}}$ has rank at most $C{-}1$, a model that drives $\Sigma_{\mathrm{within}} \to 0$ (as in neural collapse~\cite{Papyan2020nc}) concentrates essentially all variance into a $(C{-}1)$-dimensional subspace.
We define $\rho_{\mathrm{within}} = \operatorname{tr}(\Sigma_{\mathrm{within}})/\operatorname{tr}(\Sigma)$ as a scalar measure of intra-class heterogeneity (severity interpretation in \S\ref{sec:theory}).

\paragraph{Anisotropy measures.}
We quantify concentration via the \emph{participation ratio} and \emph{effective rank}:
\begin{equation}
  \mathrm{PR}(\Sigma) = \frac{\bigl(\sum_i \lambda_i\bigr)^2}{\sum_i \lambda_i^2}, \qquad
  r_{\mathrm{eff}}(\Sigma) = \exp\!\Bigl(-\sum_i \hat{\lambda}_i \log \hat{\lambda}_i\Bigr),
  \label{eq:pr-reff}
\end{equation}
where $\lambda_1 \ge \cdots \ge \lambda_d$ are eigenvalues of $\Sigma$ and $\hat{\lambda}_i = \lambda_i / \sum_j \lambda_j$.
Let $\mathcal{S}_{\mathrm{cls}} = \mathrm{span}\{v_1,\ldots,v_k\}$ denote the \emph{class-discriminative subspace} spanned by the top-$k$ eigenvectors.
Single-domain supervised models empirically exhibit extreme concentration: $\rho_k = \sum_{i=1}^{k}\lambda_i / \sum_{i=1}^{d}\lambda_i \approx \DSCRhoKLow$--$\DSCRhoKHigh$ for $k \approx C{-}1$, with $r_{\mathrm{eff}} \ll d$ (see \cref{sec:theory:empirical} and supplementary Tables~\ref{tab:supp:rank}--\ref{tab:supp:within-class-var}).

\paragraph{Subspace decomposition.}
Let $P$ denote the orthogonal projector onto $\mathcal{S}_{\mathrm{cls}}$ and $P_\perp = I - P$.
Any feature vector decomposes as $z = z_\parallel + z_\perp$ with $z_\parallel = Pz$ and $z_\perp = P_\perp z$.
The DSC regime is characterized by ID features having negligible energy outside the dominant subspace:
\begin{equation}
  \mathbb{E}\bigl[\|z_\perp\|_2^2\bigr] \le \tau^2 \quad \text{for } z = f_\theta(x_{\mathrm{ID}}).
  \label{eq:A1}
\end{equation}

\paragraph{Distance failure via variance--discriminability mismatch.}
Standard distance-based OOD scores---KNN~\cite{Sun2022knn} and Mahalanobis~\cite{Lee2018maha}---weight directions by coordinate differences in the ambient Euclidean geometry.
When ID covariance is strongly anisotropic, squared Euclidean distance is dominated by the high-variance directions in $\mathcal{S}_{\mathrm{cls}}$, regardless of their discriminative value.
We call this \emph{variance--discriminability mismatch}: directions carrying the ID-vs-OOD discrepancy have low ID variance and are therefore down-weighted by standard metrics.

\begin{theorem}[Distance failure under variance--discriminability mismatch]\label{thm:dist-fail}
Let $\lambda_1 \ge \cdots \ge \lambda_d$ be the eigenvalues of $\mathrm{Cov}(z_{\mathrm{ID}})$ with eigenvectors $v_1,\ldots,v_d$.
Suppose the ID-vs-OOD separation concentrates in a set of directions $\{v_j : j \in \mathcal{J}\}$ with $\lambda_j / \lambda_1 \le \rho$ for all $j \in \mathcal{J}$ and some $\rho \ll 1$.
Then the Euclidean KNN score distributions satisfy
\begin{equation}
  W_1\!\bigl(S_{\mathrm{KNN}}(z_{\mathrm{ID}}),\; S_{\mathrm{KNN}}(z_{\mathrm{OOD}})\bigr) \;\le\; L_k\,\varepsilon + 4L_k\,\tau^2,
  \label{eq:knn-fail}
\end{equation}
where $L_k \le 1$ is the Lipschitz constant of the $k$-NN distance statistic with respect to feature perturbations, $\varepsilon$ captures the residual ID-vs-OOD discrepancy in the high-variance subspace, and $\tau^2$ bounds the tail energy~\eqref{eq:A1}.
The coefficient $4L_k$ arises from the orthogonal-energy bound in Lemma~\ref{lem:ortho-energy}.
\end{theorem}

\begin{lemma}\label{lem:ortho-energy}
For any $z, z' \in \mathbb{R}^d$,
$\bigl|\|z - z'\|_2^2 - \|Pz - Pz'\|_2^2\bigr| = \|P_\perp(z - z')\|_2^2 \le 2\|P_\perp z\|_2^2 + 2\|P_\perp z'\|_2^2.$
Under~\eqref{eq:A1}: $\mathbb{E}\bigl[\bigl|\|z - z'\|_2^2 - \|z_\parallel - z'_\parallel\|_2^2\bigr|\bigr] \le 4\tau^2$.
\end{lemma}

\noindent\emph{Proof.} See Appendix~\ref{sec:supp:proof-thm1} for the full proof.

\paragraph{Mahalanobis distance also fails.}
The shrinkage-regularized Mahalanobis score $M_c(z) = (z - \mu_c)^\top(\Sigma_c + \lambda I)^{-1}(z - \mu_c)$ is often believed to upweight low-variance directions.
However, if both ID and OOD features have small energy outside $\mathcal{S}_{\mathrm{cls}}$ (i.e.,~\eqref{eq:A1} holds for both), then the $\mathcal{S}_{\mathrm{cls}}^\perp$ contribution is bounded by $(1/\lambda)\|P_\perp(z - \mu_c)\|_2^2$, which remains $O(\tau^2/\lambda)$.
Shrinkage caps the amplification, so Mahalanobis---and any post-hoc metric including ZCA whitening, which is algebraically equivalent to Mahalanobis with the global covariance estimate---cannot leverage the null space when condition~\eqref{eq:A1} holds for both ID and OOD inputs under $f_\theta$.
We do not claim that no scoring rule could recover the signal; rather, the whole class of distance-based scores operating on the same low-rank features fails for the same geometric reason.

\paragraph{Extension to logit-based scores.}
The same geometric conditions break MSP~\cite{Hendrycks2017msp} and Energy~\cite{Liu2020energy}, because the classifier head's sensitivity is confined to $\mathcal{S}_{\mathrm{cls}}$.
Let $\ell(x) = Wz(x) + b$ denote the logits.
Under neural-collapse alignment, the row space of $W$ concentrates in $\mathcal{S}_{\mathrm{cls}}$, yielding small \emph{head insensitivity}: $\|WP_\perp\|_{\mathrm{op}} \le \eta$ for small $\eta$.

\begin{proposition}[MSP/Energy insensitivity]\label{prop:logit-fail}
Under condition~\eqref{eq:A1}, dominant-subspace matching (OOD projections close to ID in $\mathcal{S}_{\mathrm{cls}}$), and head insensitivity $\|WP_\perp\|_{\mathrm{op}} \le \eta$, the induced score distributions satisfy
\begin{equation}
  W_1\!\bigl(S(\ell(x_{\mathrm{ID}})),\; S(\ell(x_{\mathrm{OOD}}))\bigr) \;\le\; \|W\|_{\mathrm{op}}\,\varepsilon + L_S\,\eta\tau,
  \label{eq:logit-fail}
\end{equation}
for $S \in \{S_{\mathrm{Energy}}, S_{\mathrm{MSP}}\}$, where $L_S \le 1$ for $S_{\mathrm{Energy}}$ (since $\|\nabla_\ell S_{\mathrm{Energy}}\|_2 = \|\mathrm{softmax}(\ell)\|_2 \le 1$) and $L_S \le 2$ for $S_{\mathrm{MSP}}$.
\end{proposition}

\noindent\emph{Proof.} See Appendix~\ref{sec:supp:proof-prop1}.

\begin{remark}[Interpretation and tightness]\label{rem:tightness}
All constants in Theorem~\ref{thm:dist-fail} and Proposition~\ref{prop:logit-fail} are expressed in terms of measurable model quantities: $L_k \le 1$, $\|W\|_{\mathrm{op}}$ (the spectral norm of the classifier head, typically $O(1)$ under standard weight decay), and $L_S \le 2$.
The bounds are informative precisely in the DSC regime: when $\varepsilon \ll 1$ (OOD data projects similarly to ID in $\mathcal{S}_{\mathrm{cls}}$) and $\tau \ll 1$ (tail energy is small).
Conversely, the bounds become vacuous when features are isotropic ($\tau$ large) or OOD data differs strongly in the class subspace ($\varepsilon$ large)---exactly when standard scores \emph{should} succeed.
We do not claim tightness: the bounds serve as a \emph{diagnostic tool} explaining \emph{why} scores fail under DSC, not as sharp guarantees.
Geometry audits (\cref{sec:theory:empirical}) confirm that the severe-DSC regime holds for 6 of 8 benchmarks; the remaining two serve as mild-DSC controls.
\end{remark}

\subsection{DSC Is Induced by Single-Domain Supervised Training}
\label{sec:theory:cause}

DSC is not an incidental artifact but an inherent consequence of the supervised objective on single-domain data.

\paragraph{Domain vs.\ class features.}
Decompose the input as $x = (x_d, x_y)$, where $x_d$ denotes \emph{domain features} (acquisition modality, texture statistics, sensor characteristics) and $x_y$ denotes \emph{class features}, with $I(x_d; x_y) \approx 0$.
Under single-domain training ($D{=}d_1$ a.s.), variation in $x_d$ is negligible, so $x_d$ carries approximately zero information about $y$.
The cross-entropy objective therefore provides little gradient incentive to preserve $x_d$-aligned directions, while regularization---explicit (weight decay, dropout) or implicit (gradient-descent bias toward simpler solutions)---discourages allocating capacity to them.
The resulting representation tends to attenuate $x_d$; domain-sensitive directions are often relegated to low-energy components outside $\mathcal{S}_{\mathrm{cls}}$.
This parallels the well-known result that $\ell_2$-regularized models discard task-irrelevant directions~\cite{Jing2022rank}; a formal linear-model derivation is in Appendix~\ref{sec:supp:toy-linear}.

\paragraph{Deep-network mechanisms.}
Several deep-network phenomena reinforce DSC:
(i)~\textbf{neural collapse}~\cite{Papyan2020nc} drives within-class variability to zero, concentrating representations in a $(C{-}1)$-simplex;
(ii)~\textbf{implicit max-margin bias} of gradient descent amplifies label-predictive directions;
(iii)~\textbf{augmentation driven} invariance (color jitter, random crop) explicitly suppresses style variation;
(iv)~the \textbf{low-rank classifier head} ($W \in \mathbb{R}^{C \times d}$) starves unused directions of gradient signal.
Taken together, these make DSC a pervasive byproduct of supervised single-domain training.

\paragraph{Fine-tuning reintroduces DSC.}
Even starting from a rich, multi-domain initialization (e.g., DINOv2~\cite{Oquab2024dinov2}), supervised fine-tuning on single-domain data progressively collapses the representation.
The effective rank $r_{\mathrm{eff}}$ decreases over training epochs while FPR@95 rises overall, confirming that DSC is induced by the training objective rather than the initial feature quality (\cref{sec:theory:empirical}).

\subsection{Empirical Validation of DSC Predictions}
\label{sec:theory:empirical}

The theory yields three testable predictions, each with a direct experimental diagnostic; full details are deferred to \cref{sec:experiments}.

\noindent\textbf{P1 (anisotropy):} Single-domain CE training yields severely anisotropic representations; tested via geometry audits ($r_{\mathrm{eff}}$, participation ratio, variance concentration); expected: $r_{\mathrm{eff}} \ll d$ and top-$k$ PCs capturing $>\DSCTopPCVarCapture$ variance.

\noindent\textbf{P2 (null-space concentration):} OOD energy concentrates in $\mathcal{S}_{\mathrm{cls}}^\perp$; tested by measuring $\|P_\perp z\|$ for ID vs.\ OOD; expected: $\mathrm{ood\_mean} > \mathrm{id\_mean}$.

\noindent\textbf{P3 (TGT repair):} TGT reverses DSC geometry; tested by CE-vs.-TGT geometry and null-space comparisons on matched splits; expected: improved $r_{\mathrm{eff}}$ and stronger null-space separation.

\Cref{fig:rank-fpr-gains-resnet50} provides a representative per-dataset visualization of this repair as effective-rank gain paired with FPR@95 reduction.

\paragraph{Summary of empirical support.}
Across 8 single-domain benchmarks (ResNet-50, $d{=}2048$), the geometry audits confirm the DSC severity landscape described in \cref{sec:theory}: the six severe-DSC datasets have top-\DSCTopPCCount\ PCs capturing $>\DSCTopPCVarCapture$ of variance, while Rock and Garbage retain high effective rank as predicted (supplementary Tables~\ref{tab:supp:rank}--\ref{tab:supp:within-class-var}).
Beyond confirming anisotropy, two new diagnostics emerge:
(i)~null-space energy measurements confirm $\mathrm{ood\_mean} > \mathrm{id\_mean}$ universally: 10/10 CE and 10/10 TGT splits for Rock and Garbage; 28/30 CE and 29/30 TGT splits for the other six datasets;
and (ii)~TGT reverses the collapse across all datasets: $r_{\mathrm{eff}}$ increases in \DSCReffIncreaseSplits\ splits (e.g., Colon $\DSCColonReffFrom \to \DSCColonReffTo$; Yoga $\DSCYogaReffFrom \to \DSCYogaReffTo$; EuroSAT $\DSCEuroSATReffFrom \to \DSCEuroSATReffTo$), and the participation ratio increases comparably.

\paragraph{Complementary teacher failure.}
Conversely, a multi-domain teacher (DINOv2) excels at out-of-domain OOD (\TeacherOnlyFaroodFPR{} FPR@95) but fails on in-domain OOD (\TeacherOnlyNearoodFPR{}) because self-supervised features carry no class-discriminative structure~\cite{yangcan}.
The teacher cannot distinguish a held-out class from a known one when both share the same visual domain.
Effective single-domain OOD detection therefore requires \emph{both} domain sensitivity and class discrimination---exactly the combination TGT provides by distilling the teacher's domain signal while retaining the cross-entropy objective.

\section{Method}
\label{sec:method}

Our method restores domain-sensitive representation geometry via Teacher-Guided Training (TGT), a training-time auxiliary loss that distills class-suppressed domain residuals from a frozen foundation model into the student backbone.

\subsection{Class-Suppressed Teacher Residuals}
\label{sec:method:residuals}

The key insight (\cref{sec:theory:empirical}) is that the frozen teacher excels at out-of-domain OOD but fails in-domain, while the single-domain student is the converse.
TGT bridges this gap: it transfers the teacher's domain sensitivity into the student at training time, so the student gains both signals without the teacher at inference.

Given a frozen multi-domain teacher $T$ (DINOv2 ViT-S/14~\cite{Oquab2024dinov2}), let $u(x) = T(x) \in \mathbb{R}^m$ denote the teacher feature for input~$x$.
On the training set, compute the per-class teacher means $\mu_c = \mathbb{E}[u(x) \mid y{=}c]$ and the global mean $\mu = \mathbb{E}[u(x)]$.
The \emph{class-discriminative teacher subspace} is spanned by the centered class means:
\begin{equation}
  U = [\mu_1{-}\mu,\;\ldots,\;\mu_C{-}\mu] \in \mathbb{R}^{m \times C}, \qquad
  P_{\mathrm{cls}} = U(U^\top U + \epsilon I)^{-1}U^\top,
  \label{eq:pcls}
\end{equation}
where $\epsilon > 0$ is a small regularizer.
The class-suppressed residual is then
\begin{equation}
  u_{\mathrm{dom}}(x) = (I - P_{\mathrm{cls}})\bigl(u(x) - \mu\bigr),
  \label{eq:udom}
\end{equation}
which projects out the between-class teacher directions and retains a residual emphasizing within-class variation---domain, style, and acquisition cues.
$u_{\mathrm{dom}}(x)$ is computed from frozen teacher features as a training-time target only.

\subsection{Teacher-Guided Training (TGT)}
\label{sec:method:tgt}

TGT augments the student with an auxiliary \emph{domain head} $h$ (a linear layer or shallow MLP) that predicts the class-suppressed teacher residual $u_{\mathrm{dom}}(x)$ from the student feature $z = f_\theta(x)$ (notation as in \cref{sec:theory:aniso}).
The combined training objective is
\begin{equation}
  \mathcal{L}_{\mathrm{TGT}} = \mathcal{L}_{\mathrm{CE}} + \lambda_{\mathrm{TGT}} \cdot \mathcal{L}_{\mathrm{domain}}, \qquad
  \mathcal{L}_{\mathrm{domain}} = \mathbb{E}_x\bigl[1 - \cos(\hat{h}(z), \hat{u}_{\mathrm{dom}}(x))\bigr],
  \label{eq:tgt}
\end{equation}
where $\mathcal{L}_{\mathrm{CE}}$ is the standard cross-entropy loss, $\lambda_{\mathrm{TGT}} > 0$ balances classification against domain-residual prediction, and $\hat{h}(z)$, $\hat{u}_{\mathrm{dom}}(x)$ denote $\ell_2$-normalized vectors.

\paragraph{Effects on representation geometry.}
TGT counteracts DSC by injecting multi-domain structure from the frozen teacher into the student's representation.
Empirically, TGT produces three measurable geometric changes:
\begin{enumerate}
  \item \textbf{Effective-rank gain}: $\Delta r_{\mathrm{eff}} = r_{\mathrm{eff}}^{\mathrm{TGT}} - r_{\mathrm{eff}}^{\mathrm{CE}}$ is substantial across datasets (typically on the order of +20 to +120), matching the gain-based view in \cref{fig:rank-fpr-gains-resnet50}.
  \item \textbf{Variance redistribution}: variance shifts from the class-aligned subspace $\mathcal{S}_{\mathrm{cls}}$ into domain-relevant directions, restoring the isotropic geometry needed for distance-based detection.
  \item \textbf{Distance-score restoration}: after TGT, plain MDS in the student space already achieves near-zero FPR@95 on some out-of-domain OOD benchmarks, confirming that the representation itself has been repaired.
\end{enumerate}
Classification accuracy is preserved because $\mathcal{L}_{\mathrm{CE}}$ remains in the objective; the domain head provides an orthogonal training signal that enriches the representation without degrading class separability.

\paragraph{Teacher-free inference.}
After training, the domain head $h$ is discarded.
The deployed model is the student backbone alone---no teacher forward pass, no auxiliary head, and no additional test-time cost.
Any standard post-hoc OOD scoring method (MDS, ViM, kNN, EBO, etc.) can be applied directly to the TGT-trained student features.

\section{Experiments}
\label{sec:experiments}

\subsection{Datasets and Setup}
\label{sec:exp:datasets}

\paragraph{Single-domain benchmarks.}
We evaluate on eight single-domain benchmarks spanning diverse visual modalities.
Four are \emph{sensor-homogeneous}: \textbf{Colon} (histopathology patches, 9 tissue types), \textbf{Tissue} (kidney-cortex microscopy, 8 sub-types), \textbf{EuroSAT} (Sentinel-2 multispectral satellite imagery, 10 land-use classes), and \textbf{Fashion} (greyscale product-catalogue apparel images, 10 classes).
Four are \emph{semantically narrow}: \textbf{Food} (web-scraped food photography, 101 categories), \textbf{Rock} (macro geological-sample photographs, 7 types), \textbf{Yoga} (web-sourced pose images, 6 posture classes), and \textbf{Garbage} (waste-sorting photographs, 6 categories).
All eight are single-domain in the sense of \cref{sec:theory}: domain-level visual
statistics are approximately class-invariant.
Six (Colon, Tissue, EuroSAT, Fashion, Food, Yoga) exhibit the severe-DSC geometric
signature ($r_{\mathrm{eff}}/C \in [0.86, 1.15]$); Rock and Garbage have high
within-class diversity ($\rho_{\mathrm{within}} > 0.77$) and correspondingly higher
effective rank ($r_{\mathrm{eff}}/C > 13$), placing them in the mild-DSC regime.
We retain both as natural controls to characterise TGT across the full DSC-severity
spectrum (see geometry audits in \cref{sec:theory:empirical}).

\paragraph{OOD evaluation protocol.}
For each benchmark we construct two OOD splits:
\begin{itemize}
  \item \textbf{In-domain OOD}: test images from held-out classes within the \emph{same} visual domain as the training data.
    This setting tests whether the model can identify classes it was not trained on while the domain (acquisition modality, texture statistics) remains constant---the harder, practically relevant regime.
  \item \textbf{Out-of-domain OOD}: test images drawn from a \emph{different} visual domain.
    This tests sensitivity to domain shift, e.g., presenting a satellite image to a histopathology model.
\end{itemize}
Following the OpenOOD evaluation protocol~\cite{Yang2022openood}, we add Chest X-ray images \cite{yang2023medmnist} as an additional out-of-domain source to probe performance outside the DINOv2 pretrained domain (see Appendix~\ref{sec:supp:datasets} for details).
The primary metric is FPR@95 (false-positive rate of ID inputs at 95\% true-positive recall of OOD inputs); AUROC is reported as a secondary metric.
All results are averaged across 8 benchmarks.

\paragraph{Backbones.}
We train and evaluate two architectures: \textbf{ResNet-50} (fine-tuned from ImageNet-pretrained weights, with and without TGT) and \textbf{DINOv2 ViT-S/14} (fine-tuned with and without TGT).
Both are run with a standard cross-entropy (CE) baseline and the TGT-augmented variant (TG ResNet50, TG DinoV2).
To evaluate an alternative DSC-mitigation strategy, we include \textbf{ResNet-50 with supervised contrastive learning}~\cite{Khosla2020supcon} (\textbf{SupCon}).

\paragraph{TGT implementation.}
The teacher is a frozen DINOv2 ViT-S/14~\cite{Oquab2024dinov2}.
The domain head $h$ is a two-layer MLP with a hidden dimension equal to the student's feature size.
We use $\lambda_{\mathrm{TGT}} = 1.0$ for both ResNet-50 and DINOv2.
Teacher residual targets are computed online for each training image.

\paragraph{Baselines and scorers.}
In addition to CE and TGT backbones, we report \textbf{SupCon} as an alternative DSC-mitigation baseline on ResNet-50.
We apply eight post-hoc OOD scorers to each backbone without modification:
Maximum Softmax Probability (MSP)~\cite{Hendrycks2017msp},
Energy-Based (EBO)~\cite{Liu2020energy},
Mahalanobis Distance Score (MDS)~\cite{Lee2018maha},
$k$-Nearest-Neighbor (kNN)~\cite{Sun2022knn},
Virtual-Logit Matching (ViM)~\cite{Wang2022vim},
ReAct~\cite{Sun2021react},
SCALE~\cite{Xu2024scale},
and Neural Collapse Inspired (NCI)~\cite{Liu2023nci}.
We additionally report a \emph{Teacher-Only MDS} oracle (MDS on frozen DINOv2 features, no student) as an upper bound on out-of-domain sensitivity.

\subsection{TGT Broadly Improves OOD Detection Across Scorers}
\label{sec:exp:main}

\paragraph{Out-of-domain OOD.}
\Cref{tab:farood} reports average FPR@95 and AUROC across all 8 benchmarks for out-of-domain OOD.
On \textbf{ResNet-50}, TGT improves all eight scorers on average: MDS drops from \FaroodMDSStdResNet{} to \FaroodMDSTGTResNet{} ($-$\FaroodMDSImprovementPPResNet{} pp), ViM from \FaroodViMStdResNet{} to \FaroodViMTGTResNet{} ($-$\FaroodViMImprovementPPResNet{} pp), kNN from \FaroodKnnStdResNet{} to \FaroodKnnTGTResNet{} ($-$\FaroodKnnImprovementPPResNet{} pp), and EBO from \FaroodEBOStdResNet{} to \FaroodEBOTGTResNet{} ($-$\FaroodEBOImprovementPPResNet{} pp). Per-dataset, logit-based scorers (EBO, MSP, SCALE) can worsen on specific benchmarks (e.g., Fashion out-of-domain) because TGT redistributes feature variance to be more isotropic, which can weaken softmax confidence separation when class-boundary geometry was already well-calibrated; distance-based scorers benefit uniformly from this same redistribution.
On \textbf{DINOv2}, gains are smaller and less uniform (e.g., aggregate kNN \FaroodKnnStdDino{} to \FaroodKnnTGTDino{}); we analyze this behavior in \cref{sec:exp:dino-limit}.
For reference, the \emph{MDS Teacher Only} oracle achieves \TeacherOnlyFaroodFPR{} FPR@95. TGT lowers MDS FPR@95 from \FaroodMDSStdResNet{} to \FaroodMDSTGTResNet{} on ResNet-50 and from \FaroodMDSStdDino{} to \FaroodMDSTGTDino{} on DINOv2.

\begin{table}[t]
  \centering
  \caption{Out-of-domain OOD detection averaged over 8 single-domain benchmarks (FPR@95\,|\,AUROC). \textbf{Bold numbers}: improved performance after Teacher-Guided Training (relative to the corresponding non-TG backbone). Lower FPR@95 and higher AUROC are better.}
  \label{tab:farood}
  \small
  \begin{adjustbox}{max width=\columnwidth}
\centering
\begin{tabular}{lrrrrrrrrrr}
\toprule
 & \multicolumn{2}{c}{DinoV2} & \multicolumn{2}{c}{\textbf{TG DinoV2}} & \multicolumn{2}{c}{ResNet50} & \multicolumn{2}{c}{\textbf{TG ResNet50}} & \multicolumn{2}{c}{SupCon} \\
\cmidrule(lr){2-3} \cmidrule(lr){4-5} \cmidrule(lr){6-7} \cmidrule(lr){8-9} \cmidrule(lr){10-11}
 & \footnotesize FPR@95 & \footnotesize AUROC & \footnotesize FPR@95 & \footnotesize AUROC & \footnotesize FPR@95 & \footnotesize AUROC & \footnotesize FPR@95 & \footnotesize AUROC & \footnotesize FPR@95 & \footnotesize AUROC \\
\midrule
Energy-Based (EBO) & 47.73 & 79.77 & \textbf{38.92} & \textbf{83.68} & 52.69 & 81.74 & \textbf{41.73} & \textbf{85.43} & - & - \\
Mahalanobis (MDS) & 22.30 & 91.36 & \textbf{21.46} & \textbf{91.99} & 25.55 & 91.21 & \textbf{13.94} & \textbf{95.57} & 32.02 & 85.85 \\
MaxSoftmax (MSP) & 52.47 & 76.54 & \textbf{44.41} & \textbf{80.48} & 51.31 & 80.99 & \textbf{42.18} & \textbf{84.32} & - & - \\
Nearest Neigh. (kNN) & 31.87 & 87.68 & \textbf{29.99} & \textbf{88.32} & 34.98 & 87.44 & \textbf{22.71} & \textbf{92.41} & 44.39 & 82.66 \\
Neural Collapse (NCI) & 43.42 & 82.47 & \textbf{41.41} & \textbf{82.87} & 44.77 & 83.67 & \textbf{28.14} & \textbf{89.95} & - & - \\
ReAct & 45.76 & 80.89 & \textbf{37.98} & \textbf{84.42} & 51.92 & 80.78 & \textbf{37.59} & \textbf{87.06} & - & - \\
SCALE & 49.73 & 79.21 & \textbf{38.52} & \textbf{84.19} & 55.47 & 81.53 & \textbf{41.85} & \textbf{85.74} & - & - \\
Virtual-Logit (ViM) & 22.60 & 91.67 & \textbf{22.21} & \textbf{91.90} & 26.43 & 90.81 & \textbf{15.65} & \textbf{95.32} & - & - \\
\midrule
MDS Teacher Only & 1.09 & 99.63 & 1.09 & 99.63 & 1.09 & 99.63 & 1.09 & 99.63 & - & - \\
\bottomrule
\end{tabular}

  \end{adjustbox}
\end{table}

\paragraph{In-domain OOD.}
\Cref{tab:nearood} shows that TGT improves in-domain OOD detection as well---it does \emph{not} trade out-of-domain gains for in-domain losses.
On ResNet-50, MDS drops from \NearoodMDSStdResNet{} to \NearoodMDSTGTResNet{} ($-$\NearoodMDSImprovementPPResNet{} pp) and ViM from \NearoodViMStdResNet{} to \NearoodViMTGTResNet{} ($-$\NearoodViMImprovementPPResNet{} pp).
On DINOv2, ViM improves from \NearoodViMStdDino{} to \NearoodViMTGTDino{} ($-$\NearoodViMImprovementPP{} pp) and MDS from \NearoodMDSStdDino{} to \NearoodMDSTGTDino{} ($-$\NearoodMDSImprovementPPDino{} pp).
By contrast, the \emph{MDS Teacher Only} oracle---while near-perfect for out-of-domain OOD at \TeacherOnlyFaroodFPR{}---reaches \TeacherOnlyNearoodFPR{} FPR@95 on in-domain OOD, confirming that deploying the teacher directly \emph{cannot} address the in-domain setting.
TGT is the only approach that improves both regimes simultaneously, though individual dataset--scorer combinations may regress (particularly logit-based scores and tissue/rock datasets), and DINOv2 gains are less consistent.

\noindent\textbf{Note on absolute in-domain performance.}
Despite relative gains, absolute in-domain FPR@95 remains high across all methods.
Same-domain held-out classes share acquisition modality and visual style with training classes, creating inherent feature overlap that geometry repair cannot eliminate.
We note that benchmarks such as OpenOOD~\cite{Yang2022openood} emphasize multi-domain ID sets, under-representing this harder regime; in-domain OOD detection in the single-domain setting should be regarded as largely unsolved.

\begin{table}[t]
  \centering
  \caption{In-domain OOD detection averaged over 8 single-domain benchmarks (FPR@95\,|\,AUROC). \textbf{Bold numbers}: improved performance after Teacher-Guided Training (relative to the corresponding non-TG backbone). Lower FPR@95 and higher AUROC are better.}
  \label{tab:nearood}
  \small
  \begin{adjustbox}{max width=\columnwidth}
\centering
\begin{tabular}{lrrrrrrrrrr}
\toprule
 & \multicolumn{2}{c}{DinoV2} & \multicolumn{2}{c}{\textbf{TG DinoV2}} & \multicolumn{2}{c}{ResNet50} & \multicolumn{2}{c}{\textbf{TG ResNet50}} & \multicolumn{2}{c}{SupCon} \\
\cmidrule(lr){2-3} \cmidrule(lr){4-5} \cmidrule(lr){6-7} \cmidrule(lr){8-9} \cmidrule(lr){10-11}
 & \footnotesize FPR@95 & \footnotesize AUROC & \footnotesize FPR@95 & \footnotesize AUROC & \footnotesize FPR@95 & \footnotesize AUROC & \footnotesize FPR@95 & \footnotesize AUROC & \footnotesize FPR@95 & \footnotesize AUROC \\
\midrule
Energy-Based (EBO) & 70.66 & 72.56 & \textbf{65.29} & \textbf{75.79} & 73.92 & 72.88 & \textbf{67.48} & \textbf{76.98} & - & - \\
Mahalanobis (MDS) & 68.30 & 72.69 & \textbf{64.17} & \textbf{73.97} & 70.12 & 72.35 & \textbf{59.07} & \textbf{77.57} & 80.76 & 60.15 \\
MaxSoftmax (MSP) & 73.00 & 71.42 & \textbf{66.20} & \textbf{74.06} & 73.01 & 73.02 & \textbf{66.72} & \textbf{76.62} & - & - \\
Nearest Neigh. (kNN) & 66.94 & 72.64 & \textbf{65.99} & 72.32 & 70.28 & 72.18 & \textbf{68.39} & \textbf{72.93} & 64.68 & 72.68 \\
Neural Collapse (NCI) & 73.38 & 70.47 & \textbf{68.41} & \textbf{72.83} & 76.48 & 70.44 & \textbf{71.26} & \textbf{74.84} & - & - \\
ReAct & 72.41 & 72.65 & \textbf{65.58} & \textbf{75.79} & 72.99 & 71.84 & \textbf{71.46} & \textbf{74.45} & - & - \\
SCALE & 71.25 & 72.81 & \textbf{69.04} & \textbf{74.83} & 74.71 & 72.70 & \textbf{69.51} & \textbf{76.76} & - & - \\
Virtual-Logit (ViM) & 64.10 & 75.26 & \textbf{59.07} & \textbf{76.61} & 70.12 & 73.70 & \textbf{60.68} & \textbf{78.03} & - & - \\
\midrule
MDS Teacher Only & 81.21 & 62.31 & 81.21 & 62.31 & 81.21 & 62.31 & 81.21 & 62.31 & - & - \\
\bottomrule
\end{tabular}

  \end{adjustbox}
\end{table}

\paragraph{SupCon does not alleviate DSC.}
\Cref{tab:farood,tab:nearood} show that supervised contrastive training is not a reliable remedy for the DSC failure mode in our single-domain setting.
SupCon underperforms the CE baseline on both distance-based scorers (MDS: 32.02 vs. 25.55; kNN: 44.39 vs. 34.98), inconsistent with recovering domain-sensitive structure.
On in-domain OOD, SupCon is mixed (kNN improves to 64.68 from 70.28) but remains substantially worse on MDS (80.76 vs. 70.12 for CE and 59.07 for TG ResNet50).
SupCon may improve class compactness but does not consistently restore domain-sensitive directions.

\subsection{Why DINOv2 Improves Less with TGT}
\label{sec:exp:dino-limit}

The weaker TGT gains on DINOv2 are consistent with limited teacher--student complementarity.
Both teacher and student are DINOv2, sharing architecture and pre-training initialization.
Because of this shared initialization, teacher residual targets can be highly aligned with the student's representation, reducing the novelty of the auxiliary supervision.
This can inhibit useful representation movement: the teacher objective may mainly reinforce existing structure rather than inject new geometric information.
As a result, TGT provides less of an ``ensemble-like'' benefit than in the ResNet-50 case, where a different architecture and optimization trajectory make DINOv2 distillation genuinely complementary.

\subsection{Ablations}
\label{sec:exp:ablations}

\paragraph{Classification accuracy.}
TGT improves average classification accuracy overall; where it does not improve, changes are typically within 0.1\%, with Yoga as the clear outlier ($-$4.1 pp).
Detailed per-dataset accuracy results and discussion are provided in Appendix~\ref{sec:supp:accuracy}.

\paragraph{TGT vs.\ CE geometry.}
Geometry audits are consistent with the DSC predictions in \cref{sec:theory:empirical}, and detailed split-level rank/null-space statistics are deferred to the supplementary.
\Cref{fig:rank-fpr-gains-resnet50} provides the key ablation view.

\paragraph{Different values of $\lambda$.}
See \Cref{fig:eurosat-lambda-ablation-farood} and Appendix~\ref{sec:supp:eurosat-lambda-ablation}.

\begin{figure*}[t]
  \centering
  \includegraphics[width=0.95\textwidth]{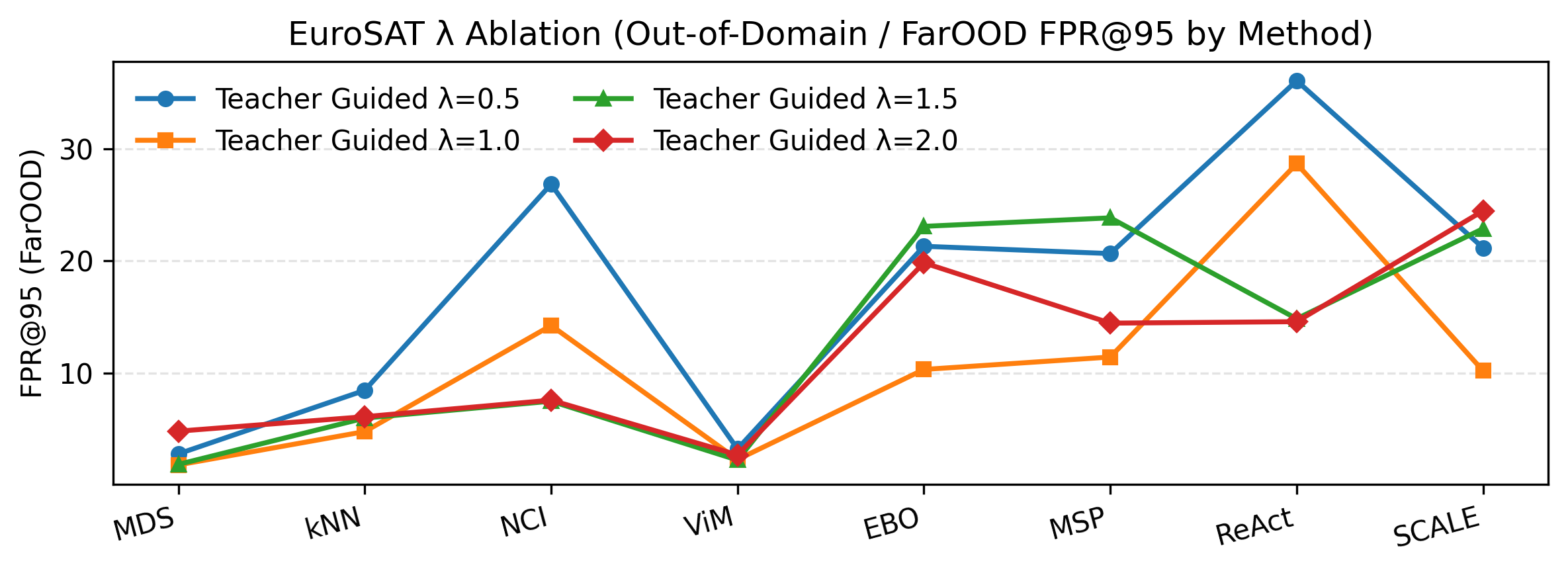}
  \caption{EuroSAT out-of-domain OOD (FarOOD) FPR@95 by method across teacher-guidance strengths $\lambda$, averaged over 5 random splits. The effect of $\lambda$ is method-dependent: for example, ReAct improves at higher $\lambda$, while SCALE worsens. Many methods also perform poorly at $\lambda=0.5$.}
  \label{fig:eurosat-lambda-ablation-farood}
\end{figure*}

\section{Conclusion}
\label{sec:conclusion}

We study OOD detection in the single-domain regime, where standard post-hoc detectors frequently fail despite strong results on multi-domain benchmarks. We trace this to \emph{Domain-Sensitivity Collapse} (DSC): supervised single-domain training compresses representations into a class-aligned low-rank subspace, suppressing domain-shift signal that post-hoc scorers depend on.

We introduce \emph{Teacher-Guided Training} (TGT), which uses class-suppressed residual supervision from a frozen multi-domain teacher to preserve domain-aware structure during training---requiring no OOD samples and adding no inference overhead. Empirically, TGT improves OOD detection on average across eight single-domain benchmarks, most reliably for distance-based scorers in both settings; logit-based gains vary per dataset. Improvements are smaller for DINOv2 fine-tuning, where shared architecture and pre-training reduce teacher--student complementarity. These results hold while preserving strong classification performance.

These results suggest that reliable OOD detection in single-domain systems is primarily a representation-learning problem, not only a scoring-rule problem. Future work includes extending DSC to other architectures and modalities, designing adaptive teacher residual targets, and studying when additional supervision improves difficult fine-grained in-domain OOD scenarios.

\appendix

\setcounter{table}{0}
\setcounter{figure}{0}
\setcounter{equation}{0}
\renewcommand{\thetable}{S\arabic{table}}
\renewcommand{\thefigure}{S\arabic{figure}}
\renewcommand{\theequation}{S\arabic{equation}}

\noindent This appendix provides supplementary material for the main paper.
Tables, figures, and equations are prefixed with ``S'' to distinguish them from
those in the main text.

\section{Dataset Details}
\label{sec:supp:datasets}

\tabref{tab:supp:datasets} summarises all 8 single-domain benchmarks used in the
experiments.  For each dataset we list the total number of classes, the train/withheld
class split, the training / test split sizes, the image resolution, and the domain-shift
characteristic.

\begin{table}[h]
  \centering
  \caption{Dataset summary. \textbf{Tot.}: total classes. \textbf{Tr./Wh.}: training
  classes / withheld classes (withheld classes serve as in-domain OOD targets; values
  from training config YAMLs). $N_{\text{train}}$ and $N_{\text{test}}$ are approximate.
  For datasets not divisible by~3 (Tissue, EuroSAT, Fashion, Rock) the split
  approximates the $\tfrac{1}{3}$-withheld protocol.}
  \label{tab:supp:datasets}
  \small
  \setlength{\tabcolsep}{4pt}
  \resizebox{\linewidth}{!}{\begin{tabular}{@{}lrrrrrlp{4.2cm}@{}}
    \toprule
    \textbf{Dataset} & \textbf{Tot.} & \textbf{Tr./Wh.}
      & $N_{\text{train}}$ & $N_{\text{test}}$
      & \textbf{Res.} & \textbf{Domain characteristic} \\
    \midrule
    Colon (PathMNIST)~\cite{yang2023medmnist}     &  9 & 6/3   &  67,497 & 22,499 & $64\times64$   & Sensor-homogeneous: histopathology tiles \\
    Tissue (TissueMNIST)~\cite{yang2023medmnist}  &  8 & 5/3   & 124,099 & 41,367 & $64\times64$   & Sensor-homogeneous: microscopy tissue slices \\
    EuroSAT~\cite{helber2019eurosat}               & 10 & 7/3   &  20,250 &  6,750 & $64\times64$   & Sensor-homogeneous: Sentinel-2 satellite \\
    Fashion (FashionMNIST)~\cite{fashion}          & 10 & 7/3   &  45,000 & 15,000 & $28\times28$   & Sensor-homogeneous: greyscale product images \\
    Food-101~\cite{food}                           & 101 & 67/34 &  75,750 & 25,250 & $224\times224$ & Semantically narrow: food photography \\
    Rock (Rock Gallery)~\cite{rock_data}           &  7 & 5/2   &   1,551 &    518 & $224\times224$ & Semantically narrow: macro rock photography \\
    Yoga (Yoga)~\cite{yoga_data}                   &  6 & 4/2   &   3,576 &  1,193 & $224\times224$ & Semantically narrow: web-sourced pose images \\
    Garbage (RealWaste)~\cite{single2023realwaste} &  9 & 6/3   &   4,752 &  1,188 & $224\times224$ & Semantically narrow: waste / recycling items \\
    \bottomrule
  \end{tabular}}
\end{table}

\begin{figure*}[htbp]
  \centering
  \includegraphics[width=\linewidth]{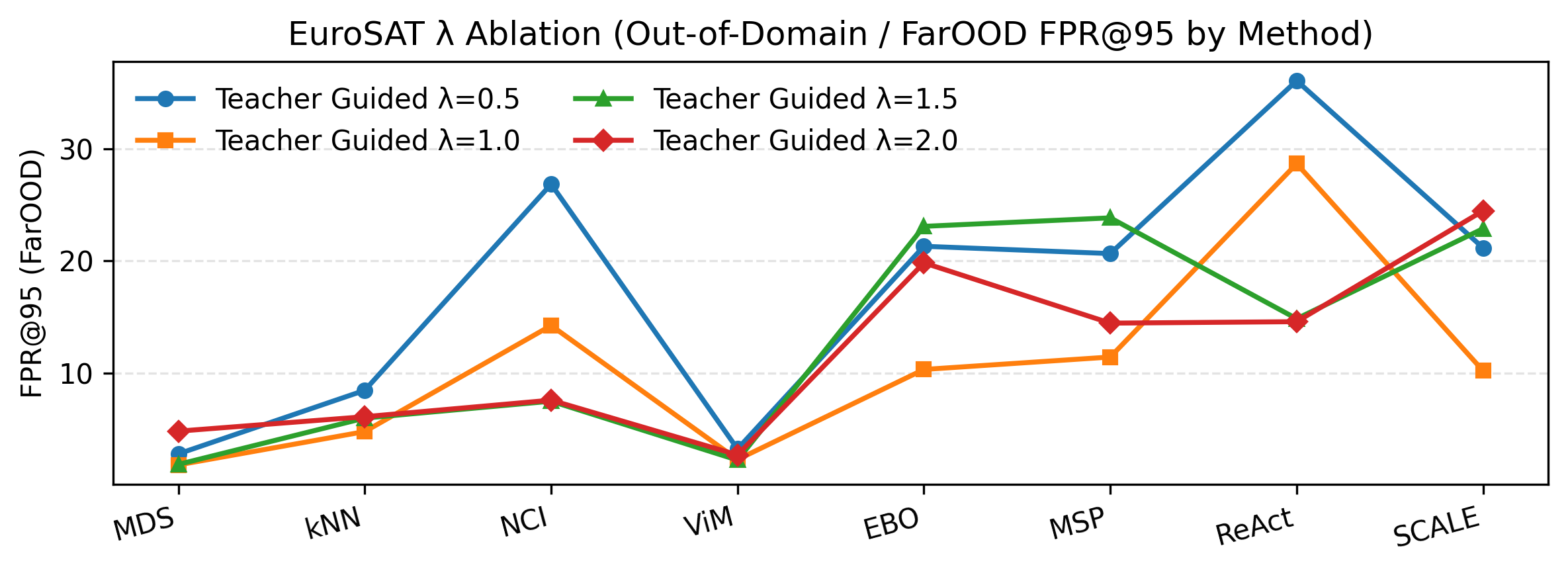}
  \caption{EuroSAT out-of-domain (FarOOD) FPR@95 by method across teacher-guidance strengths $\lambda$, averaged over 5 splits. The effect of $\lambda$ is method-dependent: for example, React improves at higher $\lambda$, while SCALE worsens. Many methods also perform poorly at $\lambda=0.5$.}
  \label{fig:supp:eurosat-lambda-ablation:farood-fpr95}
\end{figure*}

\paragraph{OOD test sets.}
For each dataset, we construct in-domain OOD using held-out classes from the same visual domain
(the OpenOOD ``nearood'' adjacent split for that dataset). Out-of-domain OOD follows the standard
OpenOOD far-OOD benchmark configuration used in our evaluation pipeline. Concretely, the far-OOD
pool contains the standard OpenOOD sets MNIST, SVHN, DTD/Textures, Places365, CIFAR-10,
CIFAR-100, and Tiny-ImageNet (TIN), plus an additional ChestMNIST set. ChestMNIST is not part of
the standard far-OOD suite and is included to assess performance on data beyond the DINOv2 training
domain. All test sets are balanced and have no overlap with training data.

\paragraph{Preprocessing.}
All images are resized to the resolution listed in \tabref{tab:supp:datasets} and
normalised using ImageNet mean and standard deviation $(0.485,0.456,0.406)$ and
$(0.229,0.224,0.225)$.  For greyscale datasets (Colon, Tissue, Fashion) channels are
replicated three times before normalisation.  All other datasets (EuroSAT, Food-101, Rock,
Yoga, Garbage) provide native RGB images.

\section{Extended Proofs}
\label{sec:supp:proofs}

This section provides full proofs for the theoretical results stated in the main paper.

\subsection{Linear-Model Illustration (Main-Text Companion)}
\label{sec:supp:toy-linear}

For completeness, we record the linear objective used in the main-text illustration.
Let $y \in \{\pm 1\}$, $x = y\,a + s$ with class direction $a \in \mathbb{R}^p$, nuisance $s \perp y$, and representation $z = Ax$.
The regularized logistic objective is
\begin{equation}
  \min_{A,w}\;\mathbb{E}\bigl[\log(1 + e^{-y\,w^\top\!Ax})\bigr] + \tfrac{\lambda}{2}(\|w\|_2^2 + \|A\|_F^2).
  \label{eq:supp:toy-loss}
\end{equation}
An optimum $A^\star$ can be chosen with row space in $\mathrm{span}(a)$, because the component orthogonal to $a$ carries no label signal while increasing the regularizer.
Hence for an OOD shift $\delta \in \mathrm{span}(a)^\perp$, one has $z_{\mathrm{OOD}} - z_{\mathrm{ID}} = A^\star\delta \approx 0$.

\subsection{Proof of Theorem~1 (Distance Failure under Variance--Discriminability Mismatch)}
\label{sec:supp:proof-thm1}

We restate the theorem for convenience.

\begin{theorem}[Distance failure]\label{thm:supp:dist-fail}
Let $\lambda_1 \ge \cdots \ge \lambda_d$ be the eigenvalues of $\mathrm{Cov}(z_{\mathrm{ID}})$
with eigenvectors $v_1,\ldots,v_d$.
Suppose the ID-vs-OOD separation concentrates in a set of directions
$\{v_j : j \in \mathcal{J}\}$ with $\lambda_j / \lambda_1 \le \rho$ for all $j \in \mathcal{J}$ and some $\rho \ll 1$.
Then the Euclidean KNN score distributions satisfy
\begin{equation}
  W_1\!\bigl(S_{\mathrm{KNN}}(z_{\mathrm{ID}}),\; S_{\mathrm{KNN}}(z_{\mathrm{OOD}})\bigr)
  \;\le\; L_k\,\varepsilon + 4L_k\,\tau^2,
  \label{eq:supp:knn-fail}
\end{equation}
where $L_k \le 1$ is the Lipschitz constant of the $k$-NN distance statistic, $\varepsilon$ captures the residual ID-vs-OOD discrepancy in the high-variance
subspace, and $\tau^2$ bounds the tail energy outside $\mathcal{S}_{\mathrm{cls}}$.
\end{theorem}

\begin{proof}
Write the $k$-nearest-neighbour score for a test point $z$ as
$S_k(z) = \|z - z_{(k)}\|_2$,
where $z_{(k)}$ denotes the $k$-th closest training point in Euclidean distance.
$S_k$ is $1$-Lipschitz in $z$ (for fixed training set).

\textbf{Step 1: Projection bound (Lemma~1 of main paper).}
For any two vectors $a, b \in \mathbb{R}^d$,
$\bigl|\|a-b\|_2 - \|Pa - Pb\|_2\bigr| \le \|P_\perp(a-b)\|_2$.
Under assumption~(3) of the main paper, $\mathbb{E}[\|P_\perp z\|_2^2] \le \tau^2$ for
ID samples; the same assumption applied to OOD gives a bound $\tau^2_{\mathrm{OOD}}$.
Hence replacing full-space distances by $\mathcal{S}_{\mathrm{cls}}$-projected distances
changes the KNN score by at most $O(\tau + \tau_{\mathrm{OOD}})$ in $L_1$.

\textbf{Step 2: Matching within $\mathcal{S}_{\mathrm{cls}}$.}
Let $\delta_{\mathcal{J}} = \mathbb{E}_{j \in \mathcal{J}}[\|Pz_{\mathrm{ID}} - Pz_{\mathrm{OOD}}\|]$
denote the expected projected distance between ID and OOD within the class-discriminative
subspace.
By hypothesis, the main separation is in the low-$\lambda_j$ directions, so
$\delta_{\mathcal{J}} \le \varepsilon$ (the residual in the dominant subspace is small).

\textbf{Step 3: Wasserstein bound.}
From Steps 1--2, the distributions of $S_k(z_{\mathrm{ID}})$ and $S_k(z_{\mathrm{OOD}})$
are within $c_1\varepsilon$ of each other in $W_1$ (from the projected-space matching),
plus a perturbation of $c_2\tau^2$ from the null-space residual.
Combining gives~\eqref{eq:supp:knn-fail} with $L_k \le 1$ (1-Lipschitz property of $k$-NN distance)
and the coefficient $4L_k$ arising from the orthogonal-energy bound in Lemma~1 of the main paper.
\end{proof}

\paragraph{Explicit constant characterization.}
The constant $L_k \le 1$ follows because the $k$-NN distance $S_k(z) = \|z - z_{(k)}\|_2$ satisfies
$|S_k(z) - S_k(z')| \le \|z - z'\|_2$ by the triangle inequality (for fixed training set).
The coefficient $4$ arises from Lemma~1: for any pair $(z, z')$ with
$\mathbb{E}[\|P_\perp z\|^2] \le \tau^2$ and $\mathbb{E}[\|P_\perp z'\|^2] \le \tau^2$,
the expected squared-distance perturbation from the null space is bounded by $4\tau^2$.
In the regime where $\varepsilon \approx 0$ (OOD data projects similarly to ID in $\mathcal{S}_{\mathrm{cls}}$)
and $\tau^2 \ll 1$ (strong DSC), the bound becomes $W_1 \le 4\tau^2$, which is small.
Conversely, when DSC is absent ($\tau \approx \sqrt{d}$), the bound is vacuous, as expected.

\subsection{Proof of Proposition~1 (MSP/Energy Insensitivity)}
\label{sec:supp:proof-prop1}

\begin{proof}
Let $S_{\mathrm{Energy}}(\ell) = -T\log\sum_{c}\exp(\ell_c/T)$ (with $T=1$ in the main
paper).
The gradient satisfies $\nabla_\ell S_{\mathrm{Energy}} = -\mathrm{softmax}(\ell)$,
so $\|\nabla_\ell S_{\mathrm{Energy}}\|_2 \le \|\mathrm{softmax}(\ell)\|_2 \le 1$.
Thus $S_{\mathrm{Energy}}$ is $1$-Lipschitz in $\ell$.

The logit difference is $\ell(x_{\mathrm{OOD}}) - \ell(x_{\mathrm{ID}}) = W(z_{\mathrm{OOD}} - z_{\mathrm{ID}})$.
Decompose $z_{\mathrm{OOD}} - z_{\mathrm{ID}} = P(z_{\mathrm{OOD}}-z_{\mathrm{ID}}) + P_\perp(z_{\mathrm{OOD}}-z_{\mathrm{ID}})$.
By the dominant-subspace matching assumption, $\|P(z_{\mathrm{OOD}}-z_{\mathrm{ID}})\| \le \varepsilon$.
By head insensitivity ($\|WP_\perp\|_{\mathrm{op}} \le \eta$) and Eq.~(3) of the main paper
(applied to both ID and OOD),
$\|WP_\perp(z_{\mathrm{OOD}}-z_{\mathrm{ID}})\| \le \eta(\tau + \tau_{\mathrm{OOD}})$.
Composing with the Lipschitz bound yields
$W_1(S_{\mathrm{Energy}}(z_{\mathrm{ID}}),S_{\mathrm{Energy}}(z_{\mathrm{OOD}}))
\le \|W\|_{\mathrm{op}}\varepsilon + L_S\eta\tau$,
where $L_S \le 1$ for $S_{\mathrm{Energy}}$ and $L_S \le 2$ for $S_{\mathrm{MSP}}$.
An analogous argument applies to MSP via
$\|\nabla_\ell\mathrm{softmax}\|_\infty \le 2$.
\end{proof}

\paragraph{Explicit constant characterization.}
Both constants are directly computable from the trained model:
$\|W\|_{\mathrm{op}}$ is the largest singular value of the classifier head (typically $O(1)$ under standard weight decay),
and the Lipschitz constants $L_S$ are intrinsic to the score function ($L_S = 1$ for Energy, $L_S \le 2$ for MSP).
The head insensitivity $\eta = \|WP_\perp\|_{\mathrm{op}}$ is also measurable:
it is the spectral norm of $W$ restricted to the null space, which approaches zero under neural collapse alignment.

\section{Geometry Audit Details}
\label{sec:supp:geometry}

\tabref{tab:supp:geometry} reports the per-dataset geometry metrics (effective rank
$r_{\mathrm{eff}}$, participation ratio, variance concentration in the top-$k$ principal
components) for CE-trained vs.\ TGT-trained models on all 8 benchmarks.

\begin{table}[h]
  \centering
  \caption{Geometry audit: effective rank $r_{\mathrm{eff}}$, participation ratio (PR),
  and variance concentration in the top-64 PCs ($\rho_{64}$, in \%).
  ResNet-50 backbone ($d{=}2048$).}
  \label{tab:supp:geometry}
  \small
  \setlength{\tabcolsep}{4pt}
  \resizebox{\linewidth}{!}{\begin{tabular}{@{}lcccccc@{}}
  \toprule
  & \multicolumn{3}{c}{\textbf{CE}} & \multicolumn{3}{c}{\textbf{TGT}} \\
  \cmidrule(lr){2-4} \cmidrule(lr){5-7}
  \textbf{Dataset} & $r_{\mathrm{eff}}$ & PR & $\rho_{64}$ (\%)                   & $r_{\mathrm{eff}}$ & PR & $\rho_{64}$ (\%) \\
  \midrule
  Colon    &    5.2 &    5.0 & 100.0 &   12.1 &    6.5 &  94.6 \\
  Tissue   &    5.3 &    4.1 &  99.5 &    9.9 &    5.0 &  95.2 \\
  EuroSAT  &    8.1 &    6.6 &  99.9 &   15.9 &    7.9 &  92.6 \\
  Fashion  &    6.2 &    5.7 & 100.0 &   18.0 &    8.5 &  93.3 \\
  Food     &   65.3 &   56.8 &  98.1 &  117.1 &   71.5 &  85.3 \\
  Rock     &  307.9 &  119.9 &  53.5 &  389.3 &  177.2 &  46.0 \\
  Yoga     &    4.5 &    2.1 &  99.5 &   47.0 &   20.9 &  90.8 \\
  Garbage  &   79.0 &   30.3 &  83.9 &   64.9 &   11.5 &  81.2 \\
  \midrule
  \textit{Mean} &   60.2 &   28.8 &  91.8 &   84.3 &   38.6 &  84.9 \\
  \bottomrule
\end{tabular}
}
\end{table}

These numbers confirm the DSC theory predictions at scale:
CE training produces extreme concentration ($r_{\mathrm{eff}} \approx 5$--$31$ vs.\ the
$d{=}2048$ ambient dimension),
while TGT consistently increases the effective rank (mean improvement: $+19.7$).

\tabref{tab:supp:rank} consolidates the DSC severity characterisation by showing the effective rank and within-class variance ratio $\rho_{\mathrm{within}}$ for each dataset, and \tabref{tab:supp:within-class-var} gives the underlying within-class and total trace values used to compute $\rho_{\mathrm{within}}$.

\begin{table}[h]
  \centering
  \caption{DSC severity characterisation for all 8 benchmarks (CE-trained ResNet-50).
  $C$ is the number of ID training classes; $r_{\mathrm{eff}}$ is the effective rank of the
  learned feature covariance; $\rho_{\mathrm{within}} = \operatorname{tr}(\Sigma_{\mathrm{within}})/\operatorname{tr}(\Sigma)$
  measures intra-class heterogeneity.
  Datasets with $r_{\mathrm{eff}}/C \approx 1$ exhibit severe DSC; Rock and Garbage
  ($r_{\mathrm{eff}}/C \gg 1$) serve as within-class-diverse controls.}
  \label{tab:supp:rank}
  \small
  \setlength{\tabcolsep}{5pt}
  \resizebox{\linewidth}{!}{\begin{tabular}{@{}lrrrrl@{}}
    \toprule
    \textbf{Dataset} & $C$ & $r_{\mathrm{eff}}$ (CE) & $r_{\mathrm{eff}}$ (TGT) & $r_{\mathrm{eff}}/C$ (CE) & DSC severity \\
    \midrule
    Colon    &  6 &   5.19 &  12.09 & 0.86 & Severe \\
    Fashion  &  7 &   6.18 &  17.97 & 0.88 & Severe \\
    Food     & 67 &  65.29 & 117.13 & 0.97 & Severe \\
    Tissue   &  5 &   5.32 &   9.90 & 1.06 & Severe \\
    Yoga     &  4 &   4.53 &  46.97 & 1.13 & Severe \\
    EuroSAT  &  7 &   8.07 &  15.95 & 1.15 & Severe \\
    \midrule
    Garbage  &  6 &  79.01 &  64.88 & 13.17 & Mild \\
    Rock     &  5 & 307.92 & 389.31 & 61.58 & Mild \\
    \bottomrule
  \end{tabular}}
\end{table}

\begin{table}[h]
  \centering
  \caption{Within-class variance ratio $\rho_{\mathrm{within}}$ for CE-trained ResNet-50 features.
  $\operatorname{tr}(\Sigma_{\mathrm{within}})$ and $\operatorname{tr}(\Sigma_{\mathrm{total}})$
  are the traces of the weighted within-class and total covariance matrices, respectively.
  Datasets are sorted by $\rho_{\mathrm{within}}$ (ascending); the two datasets with
  $\rho_{\mathrm{within}} > 0.65$ (Garbage, Rock) coincide exactly with those showing
  mild DSC in \tabref{tab:supp:rank}.}
  \label{tab:supp:within-class-var}
  \small
  \setlength{\tabcolsep}{5pt}
  \resizebox{\linewidth}{!}{\begin{tabular}{@{}lrrrr@{}}
    \toprule
    \textbf{Dataset} & $\rho_{\mathrm{within}}$ & $\operatorname{tr}(\Sigma_{\mathrm{within}})$
      & $\operatorname{tr}(\Sigma_{\mathrm{total}})$ & $N_{\mathrm{train}}$ \\
    \midrule
    Colon    & 0.009 & 0.0067 & 0.7765 & 46{,}650 \\
    Food     & 0.020 & 0.0173 & 0.8726 & 50{,}250 \\
    Fashion  & 0.028 & 0.0200 & 0.7164 & 31{,}500 \\
    EuroSAT  & 0.065 & 0.0520 & 0.8020 & 14{,}625 \\
    Tissue   & 0.072 & 0.0367 & 0.5107 & 108{,}060 \\
    Yoga     & 0.651 & 0.4992 & 0.7665 &  2{,}307 \\
    Garbage  & 0.772 & 0.7058 & 0.9143 &  2{,}310 \\
    Rock     & 0.895 & 0.8315 & 0.9286 &  1{,}245 \\
    \bottomrule
  \end{tabular}}
\end{table}

\section{Implementation and Reproducibility Details}
\label{sec:supp:impl}

\paragraph{Training setup.}
Training jobs are launched via \code{run_training.sh}, which calls
\code{OpenOOD-main/train_cross_entropy.py} (ultimately \code{main.py}) with method-specific
YAMLs and CLI overrides.
We report the exact settings used for the five methods in this paper:
\begin{itemize}
  \item \textbf{resnet}: \code{resnet50_base}.\\
  Config file: \code{configs/networks/resnet50.yml}.\\
  Pipeline file: \code{configs/pipelines/train/baseline.yml}.
  SGD, lr $0.1$, momentum $0.9$, weight decay $5\times 10^{-4}$,
  batch size $256$, $300$ epochs.

  \item \textbf{traresnet}: \code{tra_resnet50_ce}.\\
  Config file: \code{configs/networks/tra_resnet50.yml}.\\
  Pipeline file: \code{configs/pipelines/train/train_tra.yml}.
  SGD, lr $0.1$, momentum $0.9$, weight decay $5\times 10^{-4}$,
  batch size $256$, $300$ epochs.
  Trainer args: \code{lambda_tra=1.0}, \code{eps=1e-4}, \code{teacher_name=dinov2_vits14}.\\
  Normalization: \code{normalize_teacher=True}, \code{normalize_student=True}.\\
  TRA loss: \code{tra_loss=cosine}.
  The TRA head is an MLP with \code{domain_hidden_dim=2048},
  \code{domain_dim=384}, and LayerNorm enabled.

  \item \textbf{dinov2}: \code{dinov2_base}.\\
  Config file: \code{configs/networks/dinov2_classifier.yml}.\\
  Pipeline file: \code{configs/pipelines/train/baseline.yml}.
  Adam, lr $10^{-4}$, weight decay $0.05$, batch size $128$, $75$ epochs;
  image preprocessing is forced to image size $224$, pre-size $256$,
  ImageNet normalization, bicubic interpolation.

  \item \textbf{tradinov2}: \code{tra_dinov2_ce}.\\
  Config file: \code{configs/networks/tra_dinov2.yml}.\\
  Pipeline file: \code{configs/pipelines/train/train_tra.yml}.
  Adam, lr $10^{-4}$, weight decay $0.05$, batch size $128$, $75$ epochs,
  with the same preprocessing overrides as \code{dinov2}.
  Trainer args match \code{traresnet} (including \code{lambda_tra=1.0});
  the TRA head uses \code{domain_hidden_dim=384} and \code{domain_dim=384}.

  \item \textbf{resnetsupcon}: \code{simclr_supcon}.\\
  Config file: \code{configs/networks/simclr.yml}.\\
  Pipeline file: \code{configs/pipelines/train/train_simclr.yml}.
  SGD, lr $0.5$, momentum $0.9$, weight decay $10^{-4}$,
  batch size $192$, $500$ epochs,
  \code{simclr_temp=0.5}, warm-up $10$ epochs.
\end{itemize}

\paragraph{TGT details.}
For \code{traresnet} and \code{tradinov2}, the teacher-guided loss is implemented in
\code{openood/trainers/tra_trainer.py} as
$\mathcal{L}=\mathcal{L}_{\mathrm{CE}}+\lambda_{\mathrm{TRA}}\mathcal{L}_{\mathrm{TRA}}$ with
$\lambda_{\mathrm{TRA}}=1.0$.
Teacher residual targets are computed online per mini-batch from a frozen
DINOv2 ViT-S/14 teacher (\code{dinov2_vits14}) using class-prototype suppression
with projection regularizer \code{eps=1e-4}; the configured TRA loss is cosine alignment.

\paragraph{Evaluation.}
Evaluation is launched via \code{run_evaluation.sh}.
It calls \code{OpenOOD-main/train_then_evaluate.py} with \code{--skip_train}.
For each dataset, both training and evaluation run over five predefined splits
(\code{dataset}, \code{dataset_1}, \code{dataset_2}, \code{dataset_3}, \code{dataset_4}).

\paragraph{Code availability.}
Code, pretrained weights, and data splits will be released upon acceptance.

\section{Classification Accuracy: CE vs.\ TGT (ResNet-50)}
\label{sec:supp:accuracy}

\tabref{tab:supp:accuracy} reports the best test-set classification accuracy achieved by
the standard cross-entropy (CE) baseline and the TGT-trained ResNet-50 on each of the
8 benchmarks.

\begin{table}[h]
  \centering
  \caption{Classification accuracy (\%) for CE-trained and TGT-trained ResNet-50 on each
  benchmark.
  $\Delta$ (pp) is the change from CE to TGT: positive values indicate TGT
  \emph{improves} classification accuracy.
  All values are the best accuracy over 3 seeds.}
  \label{tab:supp:accuracy}
  \resizebox{\linewidth}{!}{
\begin{tabular}{lrrrr}
\toprule
Dataset & Regular Acc. & TGT Acc. & Degradation (pp) & Relative Drop (\%) \\
\midrule
rock & 0.665100 & 0.741600 & -7.650000 & -11.502000 \\
food & 0.529400 & 0.577800 & -4.840000 & -9.142000 \\
garbage & 0.899700 & 0.924100 & -2.440000 & -2.712000 \\
tissue & 0.788700 & 0.790100 & -0.140000 & -0.178000 \\
colon & 0.996800 & 0.997300 & -0.050000 & -0.050000 \\
fashion & 0.983700 & 0.983000 & 0.070000 & 0.071000 \\
eurosat & 0.992700 & 0.991900 & 0.080000 & 0.081000 \\
yoga & 0.848300 & 0.807200 & 4.110000 & 4.845000 \\
\bottomrule
\end{tabular}
}
\end{table}

\paragraph{Discussion.}
On average, TGT improves classification accuracy.
TGT improves classification accuracy on five of the eight benchmarks (Rock, Food,
Garbage, Tissue, Colon), with gains as large as $\AccDeltaRockPP$\,pp on Rock and $\AccDeltaFoodPP$\,pp on
Food---likely because the richer representation learned under TGT also benefits class
separability on visually heterogeneous domains.
On Fashion and EuroSAT the change is negligible ($<0.1$\,pp in either direction), i.e., effectively unchanged.
The only notable degradation occurs on Yoga ($\AccDeltaYogaPP$\,pp), a highly constrained
pose-classification dataset where the domain-residual auxiliary loss may conflict with
the fine-grained class boundaries.
In all other cases, the classification accuracy change is within $\pm\AccMaxDegOtherPP$\,pp, supporting
the claim that the TGT auxiliary loss preserves class-level discriminability while
substantially enriching the domain-shift signal.

\section{Teacher-Only Oracle Analysis}
\label{sec:supp:teacher-oracle}

This section provides the full per-dataset breakdown for the MDS Teacher Only experiment:
running MDS directly on frozen DINOv2 ViT-S/14 features without any student involvement.
This analysis underpins the signal-orthogonality argument developed in the main paper
(Secs.~1, 3, 4, and 5).

\subsection{Per-Dataset Results}

\begin{table}[h]
  \centering
  \caption{MDS Teacher Only: per-dataset FPR@95 (\%) on frozen DINOv2 features.
  The teacher achieves near-perfect out-of-domain OOD detection but fails catastrophically
  on in-domain OOD, confirming the orthogonality of domain-shift and class-boundary signals.}
  \label{tab:supp:teacher-oracle}
  \small
  \setlength{\tabcolsep}{5pt}
  \resizebox{\linewidth}{!}{
\begin{tabular}{@{}lcc@{}}
  \toprule
  \textbf{Dataset} & \textbf{Out-of-domain OOD} & \textbf{In-domain OOD} \\
  \midrule
  Colon & 0.35 & 100.00 \\
  Tissue & $<0.1$ & 88.83 \\
  EuroSAT & 0.19 & 65.13 \\
  Fashion & 0.95 & 63.26 \\
  Food & 0.67 & 89.31 \\
  Rock & 7.20 & 87.09 \\
  Yoga & 1.70 & 80.40 \\
  Garbage & $<0.1$ & 77.98 \\
  \midrule
  \textit{Average} & 1.09 & 81.21 \\  \bottomrule
\end{tabular}
}
\end{table}

\subsection{Why Specific Datasets Fail Harder}

The severity of in-domain OOD failure varies across datasets and correlates with the
degree of visual homogeneity within the domain:

\begin{itemize}
  \item \textbf{Colon (100\% FPR@95):} All colon pathology subtypes share identical
  tissue texture at the domain level. The class-suppressed teacher residuals encode
  ``histopathology'' but retain no information about tumour grading or tissue subtype,
  yielding complete failure.

  \item \textbf{Food (89\%) and Tissue (89\%):} Food photography and tissue microscopy
  each present high visual uniformity within the domain. Held-out classes (e.g., new food
  categories or tissue types) look identical to training classes in the teacher's
  domain-level representation.

  \item \textbf{EuroSAT (65\%) and Fashion (63\%):} These domains exhibit somewhat more
  spectral or structural diversity across classes. Satellite land-use categories have
  different spectral signatures (forest vs.\ urban vs.\ water), and fashion items have
  distinct silhouettes, allowing the domain-level features to achieve partial (but still
  poor) in-domain OOD separation.
\end{itemize}

\section{EuroSAT Lambda Ablation}
\label{sec:supp:eurosat-lambda-ablation}

This section reports the generated EuroSAT lambda-ablation tables.

\begin{table*}[htbp]
  \centering
  \caption{EuroSAT in-domain (NearOOD) lambda ablation across teacher-guidance strengths $\lambda$.}
  \label{tab:supp:eurosat-lambda-ablation:nearood}
  \small
  \begingroup
  \catcode`\_=12\relax
\resizebox{\linewidth}{!}{%
\begin{tabular}{llccccc}
\toprule
Teacher Guided $\lambda$ & Method & FPR@95 & AUROC & AUPR_IN & AUPR_OUT & FPR@98 \\
\midrule
0.5 & Mahalanobis (MDS) & 37.47 & 93.17 & 96.34 & 88.64 & 51.36 \\
0.5 & Nearest Neigh. (kNN) & 44.13 & 91.31 & 95.41 & 84.74 & 61.56 \\
0.5 & Neural Collapse (NCI) & 61.06 & 85.40 & 90.88 & 79.08 & 75.65 \\
0.5 & Virtual-Logit (ViM) & 31.61 & 93.88 & 96.91 & 88.99 & 48.06 \\
0.5 & Energy-Based (EBO) & 47.20 & 91.22 & 94.84 & 85.96 & 64.30 \\
0.5 & MaxSoftmax (MSP) & 45.59 & 91.09 & 94.96 & 84.03 & 63.32 \\
0.5 & ReAct & 68.74 & 83.41 & 89.09 & 78.67 & 81.78 \\
0.5 & SCALE & 47.40 & 91.09 & 94.77 & 85.68 & 65.12 \\
1.0 & Mahalanobis (MDS) & 43.45 & 89.34 & 94.53 & 80.14 & 59.72 \\
1.0 & Nearest Neigh. (kNN) & 48.31 & 90.85 & 95.15 & 83.65 & 62.65 \\
1.0 & Neural Collapse (NCI) & 40.21 & 91.19 & 94.95 & 86.34 & 58.60 \\
1.0 & Virtual-Logit (ViM) & 27.80 & 94.54 & 97.27 & 89.69 & 45.96 \\
1.0 & Energy-Based (EBO) & 39.32 & 93.15 & 95.64 & 89.02 & 65.20 \\
1.0 & MaxSoftmax (MSP) & 36.48 & 92.92 & 95.79 & 86.76 & 61.74 \\
1.0 & ReAct & 49.51 & 89.50 & 94.17 & 84.32 & 65.31 \\
1.0 & SCALE & 40.98 & 92.87 & 95.38 & 88.63 & 65.51 \\
1.5 & Mahalanobis (MDS) & 42.84 & 89.34 & 94.69 & 80.20 & 55.90 \\
1.5 & Nearest Neigh. (kNN) & 43.92 & 89.47 & 94.72 & 81.20 & 57.93 \\
1.5 & Neural Collapse (NCI) & 26.71 & 94.41 & 97.26 & 89.17 & 39.11 \\
1.5 & Virtual-Logit (ViM) & 20.94 & 95.31 & 97.88 & 89.82 & 32.14 \\
1.5 & Energy-Based (EBO) & 34.81 & 93.63 & 96.48 & 89.07 & 51.44 \\
1.5 & MaxSoftmax (MSP) & 33.17 & 93.04 & 96.37 & 85.62 & 49.88 \\
1.5 & ReAct & 36.57 & 93.15 & 96.20 & 88.54 & 51.17 \\
1.5 & SCALE & 35.35 & 93.49 & 96.40 & 88.79 & 52.36 \\
2.0 & Mahalanobis (MDS) & 48.73 & 88.75 & 93.89 & 80.10 & 63.38 \\
2.0 & Nearest Neigh. (kNN) & 50.02 & 89.65 & 94.43 & 82.79 & 61.62 \\
2.0 & Neural Collapse (NCI) & 32.14 & 93.09 & 96.57 & 87.38 & 46.17 \\
2.0 & Virtual-Logit (ViM) & 29.28 & 94.40 & 97.20 & 89.45 & 44.88 \\
2.0 & Energy-Based (EBO) & 38.10 & 92.97 & 95.63 & 88.30 & 65.91 \\
2.0 & MaxSoftmax (MSP) & 33.73 & 92.78 & 95.75 & 86.29 & 62.70 \\
2.0 & ReAct & 49.73 & 91.06 & 94.22 & 87.02 & 65.21 \\
2.0 & SCALE & 43.88 & 92.50 & 95.21 & 87.86 & 68.62 \\
\bottomrule
\end{tabular}%
}

  \endgroup
\end{table*}

\begin{table*}[htbp]
  \centering
  \caption{EuroSAT out-of-domain (FarOOD) lambda ablation across teacher-guidance strengths $\lambda$.}
  \label{tab:supp:eurosat-lambda-ablation:farood}
  \small
  \begingroup
  \catcode`\_=12\relax
\resizebox{\linewidth}{!}{%
\begin{tabular}{llccccc}
\toprule
Teacher Guided $\lambda$ & Method & FPR@95 & AUROC & AUPR_IN & AUPR_OUT & FPR@98 \\
\midrule
0.5 & Mahalanobis (MDS) & 2.78 & 99.38 & 97.58 & 99.85 & 5.62 \\
0.5 & Nearest Neigh. (kNN) & 8.45 & 97.71 & 93.18 & 99.46 & 13.93 \\
0.5 & Neural Collapse (NCI) & 26.85 & 94.78 & 81.18 & 98.76 & 38.84 \\
0.5 & Virtual-Logit (ViM) & 3.22 & 99.33 & 97.28 & 99.85 & 6.31 \\
0.5 & Energy-Based (EBO) & 21.30 & 95.52 & 81.05 & 98.86 & 37.79 \\
0.5 & MaxSoftmax (MSP) & 20.65 & 95.27 & 81.42 & 98.71 & 36.74 \\
0.5 & ReAct & 36.08 & 92.62 & 72.93 & 98.23 & 50.50 \\
0.5 & SCALE & 21.16 & 95.50 & 81.11 & 98.85 & 37.52 \\
1.0 & Mahalanobis (MDS) & 1.80 & 99.57 & 98.45 & 99.87 & 3.80 \\
1.0 & Nearest Neigh. (kNN) & 4.77 & 98.76 & 95.84 & 99.66 & 10.08 \\
1.0 & Neural Collapse (NCI) & 14.23 & 97.27 & 90.08 & 99.32 & 22.63 \\
1.0 & Virtual-Logit (ViM) & 2.27 & 99.50 & 98.12 & 99.86 & 4.79 \\
1.0 & Energy-Based (EBO) & 10.31 & 97.60 & 90.00 & 99.39 & 20.17 \\
1.0 & MaxSoftmax (MSP) & 11.42 & 96.92 & 89.31 & 99.15 & 20.55 \\
1.0 & ReAct & 28.67 & 93.91 & 79.81 & 98.51 & 37.78 \\
1.0 & SCALE & 10.19 & 97.59 & 90.10 & 99.38 & 19.85 \\
1.5 & Mahalanobis (MDS) & 1.84 & 99.60 & 98.25 & 99.90 & 4.32 \\
1.5 & Nearest Neigh. (kNN) & 5.97 & 98.68 & 94.97 & 99.71 & 11.75 \\
1.5 & Neural Collapse (NCI) & 7.48 & 98.19 & 94.24 & 99.58 & 12.61 \\
1.5 & Virtual-Logit (ViM) & 2.23 & 99.50 & 98.09 & 99.89 & 4.51 \\
1.5 & Energy-Based (EBO) & 23.08 & 94.02 & 79.74 & 98.82 & 39.15 \\
1.5 & MaxSoftmax (MSP) & 23.84 & 93.16 & 79.45 & 98.42 & 39.18 \\
1.5 & ReAct & 14.84 & 96.90 & 88.15 & 99.36 & 25.57 \\
1.5 & SCALE & 22.87 & 94.04 & 79.93 & 98.82 & 38.82 \\
2.0 & Mahalanobis (MDS) & 4.81 & 99.01 & 96.34 & 99.75 & 8.83 \\
2.0 & Nearest Neigh. (kNN) & 6.10 & 98.67 & 95.15 & 99.67 & 12.28 \\
2.0 & Neural Collapse (NCI) & 7.57 & 98.31 & 94.88 & 99.52 & 11.74 \\
2.0 & Virtual-Logit (ViM) & 2.69 & 99.43 & 97.79 & 99.85 & 5.43 \\
2.0 & Energy-Based (EBO) & 19.83 & 95.73 & 83.01 & 98.82 & 34.71 \\
2.0 & MaxSoftmax (MSP) & 14.45 & 96.00 & 86.63 & 98.83 & 26.60 \\
2.0 & ReAct & 14.58 & 96.76 & 88.27 & 99.15 & 24.25 \\
2.0 & SCALE & 24.47 & 95.28 & 81.34 & 98.74 & 35.82 \\
\bottomrule
\end{tabular}%
}

  \endgroup
\end{table*}

\paragraph{Discussion.}
The effect of $\lambda$ is strongly method-dependent. Across both NearOOD and FarOOD,
$\lambda{=}0.5$ is often too weak: many methods improve substantially at $\lambda\in[1.0,1.5]$.
For example, React improves sharply as $\lambda$ increases (FarOOD FPR@95: $36.08\rightarrow14.58$
from $\lambda{=}0.5$ to $2.0$; NearOOD: $68.74\rightarrow36.57$ from $\lambda{=}0.5$ to $1.5$).
In contrast, SCALE does not benefit from very large $\lambda$ on FarOOD: it is best at
$\lambda{=}1.0$ (FPR@95 $10.19$) and degrades at $\lambda{=}1.5/2.0$ (FPR@95 $22.87/24.47$).
NearOOD also shows non-monotonic behavior (SCALE: $47.40\rightarrow35.35\rightarrow43.88$
for $\lambda{=}0.5\rightarrow1.5\rightarrow2.0$), reinforcing that different methods prefer
different operating points rather than a single universally optimal $\lambda$.

\section{Extended Per-Dataset Results}
\label{sec:supp:extended-tables}

This section contains the generated extended per-dataset results; see Tables~\ref{tab:supp:extended:fpr95:dino:in} to~\ref{tab:supp:extended:auprout:tra_resnet:out}.

\begingroup
\catcode`\_=12\relax

\subsection{FPR@95}

\begin{table*}[htbp]
  \centering
  \caption{FPR@95 (dino backbone). Per-dataset in-domain OOD results.}
  \label{tab:supp:extended:fpr95:dino:in}
  \small
  \resizebox{\linewidth}{!}{
\begin{tabular}{llllllllll}
\toprule
method_display & Energy-Based (EBO) & Mahalanobis (MDS) & MaxSoftmax (MSP) & Nearest Neigh. (kNN) & Neural Collapse (NCI) & ReAct & SCALE & Virtual-Logit (ViM) & whitening \\
ID Dataset &  &  &  &  &  &  &  &  &  \\
\midrule
colon & 40.45 +- 16.55 & 37.18 +- 17.98 & 50.89 +- 27.50 & 31.00 +- 15.09 & 32.03 +- 6.97 & 40.93 +- 16.56 & 33.49 +- 11.38 & 32.85 +- 15.49 & 38.41 +- 18.35 \\
eurosat & 66.84 +- 7.42 & 36.85 +- 10.06 & 75.50 +- 9.73 & 31.71 +- 6.62 & 72.59 +- 18.61 & 71.51 +- 8.06 & 70.78 +- 14.08 & 35.58 +- 9.73 & 36.89 +- 9.74 \\
fashion & 70.93 +- 8.22 & 67.59 +- 9.74 & 65.70 +- 12.19 & 68.79 +- 13.26 & 71.06 +- 9.14 & 72.37 +- 7.14 & 64.74 +- 12.10 & 68.54 +- 10.88 & 66.94 +- 9.82 \\
food & 67.91 +- 1.50 & 76.49 +- 1.55 & 69.63 +- 1.15 & 73.32 +- 1.44 & 67.70 +- 1.69 & 68.08 +- 1.39 & 71.05 +- 1.89 & 68.62 +- 1.43 & 73.24 +- 1.46 \\
garbage & 74.66 +- 11.06 & 73.12 +- 8.37 & 77.07 +- 10.50 & 79.33 +- 12.57 & 78.28 +- 10.73 & 76.09 +- 11.81 & 77.16 +- 9.16 & 73.33 +- 7.19 & 73.54 +- 8.07 \\
rock & 86.31 +- 6.79 & 87.51 +- 6.81 & 84.71 +- 6.05 & 84.20 +- 12.60 & 89.72 +- 6.72 & 85.51 +- 6.76 & 82.83 +- 4.48 & 81.80 +- 3.23 & 87.02 +- 6.68 \\
tissue & 81.56 +- 8.46 & 87.12 +- 6.65 & 83.45 +- 7.46 & 86.99 +- 9.39 & 93.05 +- 6.95 & 81.50 +- 8.89 & 82.89 +- 8.22 & 81.42 +- 9.87 & 88.10 +- 6.00 \\
yoga & 80.59 +- 11.48 & 83.58 +- 14.41 & 79.68 +- 9.87 & 80.17 +- 6.87 & 88.74 +- 9.27 & 90.53 +- 4.13 & 87.10 +- 9.47 & 74.97 +- 19.67 & 82.48 +- 13.31 \\
\bottomrule
\end{tabular}
}
\end{table*}

\begin{table*}[htbp]
  \centering
  \caption{FPR@95 (dino backbone). Per-dataset out-of-domain OOD results.}
  \label{tab:supp:extended:fpr95:dino:out}
  \small
  \resizebox{\linewidth}{!}{\begin{tabular}{llllllllll}
\toprule
method_display & Energy-Based (EBO) & Mahalanobis (MDS) & MaxSoftmax (MSP) & Nearest Neigh. (kNN) & Neural Collapse (NCI) & ReAct & SCALE & Virtual-Logit (ViM) & whitening \\
ID Dataset &  &  &  &  &  &  &  &  &  \\
\midrule
colon & 14.64 +- 18.41 & 0.06 +- 0.04 & 13.11 +- 16.71 & 1.64 +- 1.55 & 11.39 +- 19.77 & 12.49 +- 18.35 & 25.14 +- 22.37 & 0.07 +- 0.05 & 0.11 +- 0.07 \\
eurosat & 17.85 +- 6.95 & 4.34 +- 2.34 & 19.55 +- 7.29 & 22.54 +- 15.94 & 14.27 +- 10.40 & 15.94 +- 5.59 & 19.29 +- 7.71 & 3.51 +- 2.01 & 4.41 +- 2.34 \\
fashion & 46.87 +- 15.38 & 4.86 +- 2.22 & 52.15 +- 11.54 & 8.42 +- 3.03 & 32.85 +- 13.25 & 40.53 +- 13.99 & 39.94 +- 16.30 & 5.57 +- 2.88 & 5.00 +- 2.29 \\
food & 41.13 +- 1.25 & 37.32 +- 3.21 & 53.48 +- 1.21 & 48.72 +- 6.77 & 42.19 +- 2.13 & 40.25 +- 1.55 & 46.06 +- 1.74 & 37.40 +- 1.31 & 35.76 +- 2.68 \\
garbage & 56.11 +- 9.21 & 29.38 +- 7.52 & 58.89 +- 3.81 & 49.29 +- 4.01 & 53.95 +- 4.45 & 54.84 +- 6.90 & 57.66 +- 11.77 & 30.46 +- 5.68 & 29.32 +- 7.52 \\
rock & 84.21 +- 2.62 & 59.78 +- 8.73 & 83.01 +- 3.25 & 71.58 +- 3.89 & 76.71 +- 2.95 & 83.13 +- 4.00 & 82.42 +- 5.08 & 63.42 +- 4.75 & 60.11 +- 9.02 \\
tissue & 64.42 +- 21.53 & 0.16 +- 0.26 & 77.34 +- 11.77 & 1.43 +- 0.95 & 62.63 +- 9.80 & 62.54 +- 20.42 & 66.54 +- 23.62 & 0.20 +- 0.21 & 0.22 +- 0.26 \\
yoga & 62.58 +- 10.24 & 47.56 +- 27.86 & 68.75 +- 7.32 & 51.36 +- 12.07 & 59.93 +- 9.97 & 63.43 +- 8.48 & 60.74 +- 14.63 & 51.91 +- 19.39 & 46.72 +- 24.14 \\
\bottomrule
\end{tabular}
}
\end{table*}

\begin{table*}[htbp]
  \centering
  \caption{FPR@95 (resnet backbone). Per-dataset in-domain OOD results.}
  \label{tab:supp:extended:fpr95:resnet:in}
  \small
  \resizebox{\linewidth}{!}{\begin{tabular}{lllllllllll}
\toprule
method_display & Energy-Based (EBO) & Mahalanobis (MDS) & MaxSoftmax (MSP) & Nearest Neigh. (kNN) & Neural Collapse (NCI) & ReAct & SCALE & Virtual-Logit (ViM) & whitening & mds_teacher \\
ID Dataset &  &  &  &  &  &  &  &  &  &  \\
\midrule
colon & 60.96 +- 34.33 & 40.56 +- 12.85 & 59.05 +- 36.10 & 40.05 +- 23.44 & 74.51 +- 14.10 & 39.01 +- 12.14 & 63.81 +- 29.89 & 44.68 +- 23.78 & 42.22 +- 16.30 & 100.00 +- 0.00 \\
eurosat & 42.46 +- 9.93 & 51.37 +- 17.98 & 41.32 +- 10.69 & 49.04 +- 17.78 & 72.70 +- 6.48 & 55.47 +- 8.75 & 43.62 +- 9.39 & 50.35 +- 17.30 & 48.76 +- 16.30 & 65.13 +- 6.37 \\
fashion & 77.21 +- 13.39 & 75.48 +- 10.68 & 75.61 +- 13.66 & 79.07 +- 9.33 & 65.88 +- 15.38 & 85.12 +- 13.39 & 76.76 +- 15.27 & 77.52 +- 13.34 & 78.91 +- 9.32 & 63.26 +- 22.24 \\
food & 75.21 +- 2.55 & 76.75 +- 1.65 & 74.86 +- 2.26 & 73.42 +- 1.82 & 80.39 +- 1.24 & 85.89 +- 1.73 & 77.47 +- 2.99 & 75.98 +- 1.63 & 76.92 +- 1.61 & 89.31 +- 1.44 \\
garbage & 76.96 +- 11.26 & 57.54 +- 9.34 & 73.11 +- 13.20 & 69.51 +- 8.82 & 69.70 +- 9.73 & 77.16 +- 11.81 & 78.13 +- 10.56 & 58.61 +- 11.67 & 54.91 +- 9.89 & 77.98 +- 8.13 \\
rock & 85.80 +- 6.98 & 83.47 +- 3.29 & 85.25 +- 7.39 & 75.22 +- 5.91 & 76.99 +- 9.06 & 85.32 +- 9.75 & 85.40 +- 6.92 & 79.87 +- 8.63 & 79.94 +- 6.32 & 87.09 +- 9.98 \\
tissue & 84.44 +- 4.50 & 91.28 +- 6.76 & 84.23 +- 4.75 & 92.69 +- 5.08 & 84.57 +- 11.71 & 81.71 +- 10.39 & 83.78 +- 4.30 & 85.89 +- 8.50 & 92.01 +- 7.39 & 88.83 +- 4.59 \\
yoga & 75.34 +- 7.90 & 70.71 +- 13.31 & 76.02 +- 6.90 & 83.18 +- 11.72 & 82.11 +- 11.45 & 75.52 +- 7.18 & 77.09 +- 7.97 & 74.67 +- 11.61 & 73.29 +- 16.56 & 80.40 +- 7.88 \\
\bottomrule
\end{tabular}
}
\end{table*}

\begin{table*}[htbp]
  \centering
  \caption{FPR@95 (resnet backbone). Per-dataset out-of-domain OOD results.}
  \label{tab:supp:extended:fpr95:resnet:out}
  \small
  \resizebox{\linewidth}{!}{\begin{tabular}{lllllllllll}
\toprule
method_display & Energy-Based (EBO) & Mahalanobis (MDS) & MaxSoftmax (MSP) & Nearest Neigh. (kNN) & Neural Collapse (NCI) & ReAct & SCALE & Virtual-Logit (ViM) & whitening & mds_teacher \\
ID Dataset &  &  &  &  &  &  &  &  &  &  \\
\midrule
colon & 70.74 +- 7.43 & 12.67 +- 14.23 & 53.01 +- 13.28 & 13.26 +- 17.17 & 24.81 +- 12.62 & 61.22 +- 19.30 & 71.29 +- 6.78 & 9.78 +- 9.87 & 18.69 +- 21.00 & 0.35 +- 0.10 \\
eurosat & 50.09 +- 8.58 & 13.69 +- 5.05 & 49.82 +- 9.11 & 31.37 +- 8.39 & 57.03 +- 6.99 & 54.37 +- 11.91 & 50.42 +- 8.51 & 17.75 +- 6.40 & 17.17 +- 5.74 & 0.19 +- 0.05 \\
fashion & 15.78 +- 6.88 & 6.34 +- 2.23 & 19.74 +- 8.46 & 14.88 +- 6.19 & 38.59 +- 10.31 & 34.13 +- 22.15 & 19.86 +- 6.78 & 8.45 +- 3.46 & 7.23 +- 2.70 & 0.95 +- 0.60 \\
food & 62.92 +- 2.25 & 52.22 +- 2.60 & 63.69 +- 2.04 & 62.75 +- 2.01 & 65.41 +- 2.62 & 71.10 +- 2.23 & 63.47 +- 2.07 & 53.58 +- 2.45 & 52.25 +- 2.47 & 0.67 +- 0.07 \\
garbage & 35.04 +- 30.53 & 14.49 +- 18.95 & 34.74 +- 25.39 & 31.77 +- 7.52 & 30.41 +- 20.60 & 36.00 +- 22.90 & 38.91 +- 29.06 & 10.95 +- 12.59 & 12.45 +- 15.74 & 0.10 +- 0.08 \\
rock & 70.60 +- 5.85 & 46.54 +- 7.07 & 71.28 +- 4.58 & 60.39 +- 7.81 & 64.19 +- 0.84 & 74.60 +- 4.26 & 70.74 +- 5.70 & 52.66 +- 2.62 & 46.52 +- 6.87 & 7.20 +- 1.17 \\
tissue & 79.23 +- 10.15 & 3.65 +- 1.46 & 76.65 +- 11.86 & 31.54 +- 3.89 & 35.73 +- 18.62 & 48.02 +- 20.42 & 73.89 +- 11.86 & 5.03 +- 3.33 & 5.99 +- 1.65 & 0.00 +- 0.00 \\
yoga & 34.11 +- 10.69 & 30.17 +- 24.90 & 34.76 +- 10.27 & 33.91 +- 24.60 & 32.42 +- 10.92 & 35.32 +- 17.15 & 35.74 +- 13.31 & 26.76 +- 20.37 & 36.49 +- 33.44 & 1.70 +- 0.59 \\
\bottomrule
\end{tabular}
}
\end{table*}

\begin{table*}[htbp]
  \centering
  \caption{FPR@95 (simclr backbone). Per-dataset in-domain OOD results.}
  \label{tab:supp:extended:fpr95:simclr:in}
  \small
  \resizebox{\linewidth}{!}{\begin{tabular}{lll}
\toprule
method_display & Mahalanobis (MDS) & Nearest Neigh. (kNN) \\
ID Dataset &  &  \\
\midrule
colon & 68.94 +- 15.61 & 28.09 +- 13.97 \\
eurosat & 72.75 +- 18.75 & 34.13 +- 14.70 \\
fashion & 76.45 +- 13.05 & 58.80 +- 12.89 \\
food & 95.20 +- 0.73 & 65.82 +- 4.33 \\
garbage & 90.13 +- 5.26 & 84.71 +- 9.80 \\
rock & 84.91 +- 2.67 & 87.05 +- 10.25 \\
tissue & 86.18 +- 12.41 & 87.55 +- 7.49 \\
yoga & 74.39 +- 3.43 & 71.50 +- 15.70 \\
\bottomrule
\end{tabular}
}
\end{table*}

\begin{table*}[htbp]
  \centering
  \caption{FPR@95 (simclr backbone). Per-dataset out-of-domain OOD results.}
  \label{tab:supp:extended:fpr95:simclr:out}
  \small
  \resizebox{\linewidth}{!}{\begin{tabular}{lll}
\toprule
method_display & Mahalanobis (MDS) & Nearest Neigh. (kNN) \\
ID Dataset &  &  \\
\midrule
colon & 4.09 +- 2.28 & 21.98 +- 12.18 \\
eurosat & 13.95 +- 3.87 & 20.67 +- 4.70 \\
fashion & 6.61 +- 2.16 & 23.15 +- 4.84 \\
food & 89.89 +- 3.88 & 44.07 +- 3.30 \\
garbage & 6.11 +- 2.41 & 70.73 +- 8.56 \\
rock & 71.96 +- 4.59 & 84.01 +- 3.74 \\
tissue & 0.02 +- 0.02 & 54.13 +- 8.25 \\
yoga & 75.06 +- 6.03 & 36.32 +- 24.04 \\
\bottomrule
\end{tabular}
}
\end{table*}

\begin{table*}[htbp]
  \centering
  \caption{FPR@95 (tra\_dino backbone). Per-dataset in-domain OOD results.}
  \label{tab:supp:extended:fpr95:tra_dino:in}
  \small
  \resizebox{\linewidth}{!}{\begin{tabular}{llllllllll}
\toprule
method_display & Energy-Based (EBO) & Mahalanobis (MDS) & MaxSoftmax (MSP) & Nearest Neigh. (kNN) & Neural Collapse (NCI) & ReAct & SCALE & Virtual-Logit (ViM) & whitening \\
ID Dataset &  &  &  &  &  &  &  &  &  \\
\midrule
colon & 49.27 +- 29.23 & 23.98 +- 13.76 & 40.55 +- 26.52 & 24.40 +- 9.83 & 48.25 +- 21.83 & 49.83 +- 27.22 & 48.90 +- 20.43 & 24.60 +- 15.28 & 24.02 +- 13.65 \\
eurosat & 32.28 +- 11.54 & 37.05 +- 13.26 & 37.03 +- 13.37 & 38.93 +- 15.85 & 35.77 +- 11.88 & 32.42 +- 11.37 & 35.29 +- 12.09 & 26.56 +- 7.26 & 36.66 +- 12.89 \\
fashion & 56.81 +- 18.77 & 57.68 +- 17.61 & 58.14 +- 19.17 & 56.97 +- 17.44 & 57.51 +- 18.00 & 57.29 +- 19.10 & 70.15 +- 24.19 & 55.15 +- 19.47 & 57.43 +- 18.00 \\
food & 65.68 +- 1.87 & 74.09 +- 1.53 & 67.24 +- 2.06 & 72.19 +- 1.69 & 65.32 +- 2.12 & 65.70 +- 1.93 & 69.98 +- 2.34 & 66.20 +- 1.94 & 70.80 +- 1.84 \\
garbage & 74.15 +- 10.58 & 64.02 +- 2.72 & 75.95 +- 10.16 & 75.62 +- 9.78 & 78.29 +- 6.70 & 77.26 +- 14.13 & 75.76 +- 7.65 & 61.96 +- 7.56 & 63.72 +- 2.74 \\
rock & 81.68 +- 7.35 & 90.32 +- 5.55 & 87.71 +- 4.13 & 91.20 +- 7.86 & 84.67 +- 7.39 & 80.09 +- 7.64 & 83.93 +- 8.33 & 84.32 +- 8.54 & 89.97 +- 5.46 \\
tissue & 81.14 +- 9.35 & 84.63 +- 8.36 & 84.23 +- 6.89 & 83.83 +- 9.03 & 92.24 +- 2.94 & 81.50 +- 8.98 & 83.06 +- 9.24 & 79.81 +- 11.30 & 84.93 +- 8.31 \\
yoga & 81.30 +- 14.58 & 81.56 +- 9.47 & 78.78 +- 10.90 & 84.78 +- 9.92 & 85.25 +- 8.82 & 80.58 +- 15.99 & 85.22 +- 11.51 & 73.97 +- 17.18 & 81.40 +- 9.59 \\
\bottomrule
\end{tabular}
}
\end{table*}

\begin{table*}[htbp]
  \centering
  \caption{FPR@95 (tra\_dino backbone). Per-dataset out-of-domain OOD results.}
  \label{tab:supp:extended:fpr95:tra_dino:out}
  \small
  \resizebox{\linewidth}{!}{\begin{tabular}{llllllllll}
\toprule
method_display & Energy-Based (EBO) & Mahalanobis (MDS) & MaxSoftmax (MSP) & Nearest Neigh. (kNN) & Neural Collapse (NCI) & ReAct & SCALE & Virtual-Logit (ViM) & whitening \\
ID Dataset &  &  &  &  &  &  &  &  &  \\
\midrule
colon & 5.01 +- 4.82 & 0.03 +- 0.03 & 4.57 +- 4.96 & 1.08 +- 0.96 & 3.05 +- 2.84 & 4.38 +- 4.79 & 9.69 +- 10.61 & 0.07 +- 0.05 & 0.04 +- 0.04 \\
eurosat & 17.44 +- 12.77 & 8.36 +- 7.43 & 18.68 +- 12.44 & 12.70 +- 10.33 & 18.48 +- 15.56 & 16.74 +- 12.83 & 18.20 +- 11.82 & 7.46 +- 6.08 & 8.36 +- 7.43 \\
fashion & 12.10 +- 5.68 & 4.58 +- 1.42 & 19.02 +- 9.01 & 10.51 +- 3.89 & 12.43 +- 6.85 & 11.82 +- 5.69 & 11.38 +- 3.50 & 5.07 +- 2.20 & 4.56 +- 1.43 \\
food & 37.56 +- 2.49 & 36.48 +- 4.14 & 50.70 +- 2.61 & 46.42 +- 5.29 & 37.56 +- 2.97 & 36.97 +- 2.78 & 40.43 +- 4.77 & 34.94 +- 2.68 & 34.53 +- 3.94 \\
garbage & 48.56 +- 7.01 & 23.66 +- 8.63 & 53.89 +- 6.69 & 48.29 +- 8.41 & 46.23 +- 5.06 & 47.37 +- 6.04 & 46.67 +- 9.87 & 22.23 +- 5.78 & 23.74 +- 8.64 \\
rock & 81.65 +- 4.41 & 55.14 +- 5.04 & 82.10 +- 5.42 & 66.74 +- 4.20 & 78.47 +- 8.10 & 80.10 +- 4.49 & 82.55 +- 4.90 & 59.84 +- 1.64 & 55.34 +- 5.12 \\
tissue & 48.46 +- 22.52 & 0.34 +- 0.41 & 63.23 +- 18.62 & 3.99 +- 1.41 & 74.60 +- 5.63 & 46.72 +- 21.08 & 40.57 +- 20.95 & 1.47 +- 0.97 & 0.35 +- 0.40 \\
yoga & 60.55 +- 16.88 & 43.05 +- 17.09 & 63.14 +- 16.07 & 50.21 +- 9.05 & 60.41 +- 15.66 & 59.73 +- 16.85 & 58.71 +- 12.73 & 46.63 +- 16.03 & 42.98 +- 17.01 \\
\bottomrule
\end{tabular}
}
\end{table*}

\begin{table*}[htbp]
  \centering
  \caption{FPR@95 (tra\_resnet backbone). Per-dataset in-domain OOD results.}
  \label{tab:supp:extended:fpr95:tra_resnet:in}
  \small
  \resizebox{\linewidth}{!}{\begin{tabular}{llllllllll}
\toprule
method_display & Energy-Based (EBO) & Mahalanobis (MDS) & MaxSoftmax (MSP) & Nearest Neigh. (kNN) & Neural Collapse (NCI) & ReAct & SCALE & Virtual-Logit (ViM) & whitening \\
ID Dataset &  &  &  &  &  &  &  &  &  \\
\midrule
colon & 51.85 +- 32.56 & 26.37 +- 15.99 & 51.57 +- 32.47 & 30.06 +- 14.11 & 75.81 +- 8.21 & 60.42 +- 25.38 & 54.66 +- 32.66 & 32.96 +- 14.45 & 28.10 +- 17.04 \\
eurosat & 39.36 +- 14.59 & 43.45 +- 9.39 & 36.51 +- 9.20 & 47.99 +- 4.39 & 40.21 +- 25.57 & 49.49 +- 21.68 & 41.01 +- 17.05 & 27.80 +- 7.58 & 43.64 +- 9.08 \\
fashion & 73.33 +- 13.59 & 64.35 +- 13.02 & 71.21 +- 14.81 & 61.19 +- 15.98 & 71.23 +- 10.34 & 70.32 +- 15.47 & 70.29 +- 20.73 & 62.64 +- 12.74 & 64.12 +- 13.57 \\
food & 70.84 +- 1.22 & 74.40 +- 0.65 & 71.29 +- 1.36 & 71.12 +- 1.42 & 70.73 +- 1.01 & 82.66 +- 4.92 & 80.85 +- 10.34 & 71.49 +- 0.74 & 73.72 +- 0.70 \\
garbage & 68.00 +- 8.49 & 53.21 +- 12.01 & 64.46 +- 9.51 & 74.10 +- 8.80 & 69.47 +- 4.07 & 70.61 +- 13.95 & 68.10 +- 8.50 & 55.07 +- 8.06 & 49.40 +- 9.16 \\
rock & 76.79 +- 6.87 & 83.91 +- 5.79 & 78.84 +- 10.07 & 79.51 +- 4.91 & 72.82 +- 3.21 & 81.75 +- 8.11 & 76.79 +- 6.87 & 75.75 +- 6.09 & 83.71 +- 7.19 \\
tissue & 84.71 +- 7.16 & 93.46 +- 3.72 & 84.14 +- 7.46 & 94.39 +- 2.05 & 81.95 +- 9.44 & 79.20 +- 11.42 & 87.98 +- 4.46 & 84.23 +- 11.62 & 93.63 +- 3.51 \\
yoga & 74.98 +- 6.11 & 66.11 +- 18.74 & 75.73 +- 4.75 & 88.76 +- 11.13 & 87.86 +- 12.75 & 77.20 +- 4.95 & 76.37 +- 7.48 & 75.50 +- 14.25 & 63.42 +- 17.75 \\
\bottomrule
\end{tabular}
}
\end{table*}

\begin{table*}[htbp]
  \centering
  \caption{FPR@95 (tra\_resnet backbone). Per-dataset out-of-domain OOD results.}
  \label{tab:supp:extended:fpr95:tra_resnet:out}
  \small
  \resizebox{\linewidth}{!}{\begin{tabular}{llllllllll}
\toprule
method_display & Energy-Based (EBO) & Mahalanobis (MDS) & MaxSoftmax (MSP) & Nearest Neigh. (kNN) & Neural Collapse (NCI) & ReAct & SCALE & Virtual-Logit (ViM) & whitening \\
ID Dataset &  &  &  &  &  &  &  &  &  \\
\midrule
colon & 25.82 +- 22.98 & 0.73 +- 0.77 & 21.60 +- 22.98 & 2.25 +- 1.31 & 5.91 +- 8.80 & 27.21 +- 3.98 & 30.03 +- 24.21 & 0.90 +- 1.08 & 0.75 +- 0.83 \\
eurosat & 10.31 +- 1.83 & 1.80 +- 1.01 & 11.42 +- 2.27 & 4.73 +- 1.39 & 14.23 +- 11.44 & 28.68 +- 21.87 & 10.19 +- 1.94 & 2.27 +- 1.26 & 1.81 +- 1.02 \\
fashion & 53.68 +- 31.32 & 2.66 +- 2.23 & 53.70 +- 29.21 & 8.45 +- 5.85 & 8.98 +- 2.51 & 23.81 +- 14.44 & 55.54 +- 30.18 & 2.51 +- 2.06 & 2.49 +- 2.12 \\
food & 51.86 +- 3.62 & 35.16 +- 4.33 & 56.13 +- 2.84 & 50.50 +- 3.47 & 49.95 +- 3.05 & 47.79 +- 5.02 & 52.50 +- 2.47 & 35.93 +- 4.96 & 35.11 +- 4.37 \\
garbage & 33.45 +- 34.16 & 17.93 +- 33.96 & 34.98 +- 29.15 & 29.16 +- 19.62 & 33.26 +- 19.17 & 32.95 +- 24.25 & 33.75 +- 34.20 & 12.34 +- 11.50 & 11.48 +- 20.70 \\
rock & 61.34 +- 10.53 & 33.87 +- 6.66 & 63.75 +- 8.63 & 40.36 +- 5.56 & 57.71 +- 8.03 & 65.70 +- 10.27 & 61.38 +- 10.32 & 43.66 +- 8.15 & 33.82 +- 6.22 \\
tissue & 61.52 +- 10.82 & 0.90 +- 0.88 & 59.78 +- 11.49 & 17.86 +- 3.37 & 20.41 +- 5.80 & 41.40 +- 27.70 & 53.96 +- 16.87 & 0.40 +- 0.43 & 0.87 +- 0.89 \\
yoga & 35.86 +- 11.80 & 31.69 +- 27.32 & 36.05 +- 12.17 & 28.34 +- 18.13 & 34.67 +- 16.49 & 33.15 +- 15.16 & 37.46 +- 13.84 & 27.18 +- 15.95 & 31.40 +- 27.54 \\
\bottomrule
\end{tabular}
}
\end{table*}

\subsection{AUROC}

\begin{table*}[htbp]
  \centering
  \caption{AUROC (dino backbone). Per-dataset in-domain OOD results.}
  \label{tab:supp:extended:auroc:dino:in}
  \small
  \resizebox{\linewidth}{!}{\begin{tabular}{llllllllll}
\toprule
method_display & Energy-Based (EBO) & Mahalanobis (MDS) & MaxSoftmax (MSP) & Nearest Neigh. (kNN) & Neural Collapse (NCI) & ReAct & SCALE & Virtual-Logit (ViM) & whitening \\
ID Dataset &  &  &  &  &  &  &  &  &  \\
\midrule
colon & 89.27 +- 6.85 & 91.12 +- 4.52 & 87.50 +- 7.47 & 91.78 +- 4.78 & 92.12 +- 2.73 & 89.23 +- 6.76 & 92.15 +- 2.71 & 92.45 +- 3.78 & 90.65 +- 4.77 \\
eurosat & 85.27 +- 2.73 & 91.22 +- 2.38 & 84.56 +- 3.79 & 92.05 +- 1.84 & 84.78 +- 4.69 & 84.78 +- 3.51 & 83.25 +- 5.07 & 91.66 +- 2.37 & 91.17 +- 2.39 \\
fashion & 70.03 +- 10.97 & 79.29 +- 6.20 & 73.55 +- 9.92 & 80.09 +- 7.45 & 71.55 +- 12.12 & 72.24 +- 9.24 & 75.09 +- 11.15 & 78.32 +- 7.46 & 79.44 +- 6.15 \\
food & 71.21 +- 0.57 & 66.82 +- 0.47 & 69.19 +- 0.72 & 68.40 +- 0.83 & 71.22 +- 0.66 & 71.20 +- 0.52 & 70.24 +- 0.83 & 70.55 +- 0.66 & 68.19 +- 0.43 \\
garbage & 70.59 +- 8.45 & 71.72 +- 7.33 & 69.28 +- 7.44 & 62.54 +- 14.76 & 68.06 +- 8.75 & 70.53 +- 8.57 & 70.16 +- 8.28 & 73.30 +- 7.03 & 71.71 +- 7.32 \\
rock & 62.35 +- 7.25 & 57.09 +- 7.86 & 59.12 +- 7.98 & 59.89 +- 12.08 & 60.87 +- 6.32 & 63.81 +- 5.77 & 64.60 +- 6.39 & 62.28 +- 5.98 & 57.08 +- 7.76 \\
tissue & 61.04 +- 13.09 & 60.89 +- 8.24 & 58.77 +- 12.89 & 59.99 +- 7.14 & 50.79 +- 14.69 & 61.86 +- 12.55 & 60.86 +- 13.56 & 62.44 +- 13.21 & 60.20 +- 7.51 \\
yoga & 69.57 +- 7.65 & 61.06 +- 15.57 & 68.02 +- 8.14 & 66.37 +- 5.04 & 60.30 +- 12.29 & 64.10 +- 4.10 & 66.16 +- 7.03 & 68.28 +- 11.73 & 62.25 +- 13.76 \\
\bottomrule
\end{tabular}
}
\end{table*}

\begin{table*}[htbp]
  \centering
  \caption{AUROC (dino backbone). Per-dataset out-of-domain OOD results.}
  \label{tab:supp:extended:auroc:dino:out}
  \small
  \resizebox{\linewidth}{!}{\begin{tabular}{llllllllll}
\toprule
method_display & Energy-Based (EBO) & Mahalanobis (MDS) & MaxSoftmax (MSP) & Nearest Neigh. (kNN) & Neural Collapse (NCI) & ReAct & SCALE & Virtual-Logit (ViM) & whitening \\
ID Dataset &  &  &  &  &  &  &  &  &  \\
\midrule
colon & 96.98 +- 3.88 & 99.98 +- 0.02 & 96.67 +- 3.64 & 99.59 +- 0.41 & 97.35 +- 4.28 & 97.29 +- 3.84 & 94.42 +- 5.02 & 99.98 +- 0.01 & 99.97 +- 0.03 \\
eurosat & 95.64 +- 1.34 & 99.10 +- 0.46 & 94.16 +- 1.83 & 94.80 +- 3.50 & 96.13 +- 3.51 & 96.23 +- 1.04 & 95.35 +- 1.35 & 99.26 +- 0.38 & 99.08 +- 0.47 \\
fashion & 84.61 +- 6.98 & 98.67 +- 0.68 & 80.12 +- 4.25 & 97.69 +- 1.04 & 90.30 +- 5.77 & 87.62 +- 6.29 & 87.81 +- 7.40 & 98.65 +- 0.72 & 98.64 +- 0.70 \\
food & 86.21 +- 1.44 & 87.99 +- 1.90 & 78.64 +- 1.17 & 83.22 +- 3.32 & 85.37 +- 1.36 & 86.82 +- 1.16 & 83.16 +- 2.40 & 87.47 +- 1.45 & 88.55 +- 1.68 \\
garbage & 78.75 +- 6.24 & 91.67 +- 3.70 & 75.97 +- 2.18 & 82.11 +- 3.44 & 79.48 +- 1.83 & 79.49 +- 4.77 & 78.15 +- 7.50 & 91.26 +- 3.00 & 91.65 +- 3.70 \\
rock & 52.99 +- 1.94 & 74.92 +- 5.46 & 54.60 +- 5.74 & 66.31 +- 4.31 & 63.87 +- 2.34 & 54.96 +- 3.18 & 55.01 +- 3.02 & 74.21 +- 4.75 & 74.81 +- 5.44 \\
tissue & 67.57 +- 20.50 & 99.95 +- 0.06 & 60.12 +- 15.00 & 99.68 +- 0.19 & 69.71 +- 9.78 & 70.74 +- 17.66 & 66.08 +- 19.41 & 99.90 +- 0.09 & 99.93 +- 0.06 \\
yoga & 72.50 +- 9.45 & 75.42 +- 23.42 & 68.99 +- 6.96 & 78.04 +- 8.89 & 74.30 +- 10.25 & 69.37 +- 8.14 & 73.71 +- 10.72 & 76.60 +- 14.92 & 77.13 +- 20.61 \\
\bottomrule
\end{tabular}
}
\end{table*}

\begin{table*}[htbp]
  \centering
  \caption{AUROC (resnet backbone). Per-dataset in-domain OOD results.}
  \label{tab:supp:extended:auroc:resnet:in}
  \small
  \resizebox{\linewidth}{!}{\begin{tabular}{lllllllllll}
\toprule
method_display & Energy-Based (EBO) & Mahalanobis (MDS) & MaxSoftmax (MSP) & Nearest Neigh. (kNN) & Neural Collapse (NCI) & ReAct & SCALE & Virtual-Logit (ViM) & whitening & mds_teacher \\
ID Dataset &  &  &  &  &  &  &  &  &  &  \\
\midrule
colon & 84.74 +- 12.53 & 89.28 +- 5.48 & 85.03 +- 12.46 & 89.43 +- 8.11 & 68.30 +- 16.17 & 86.31 +- 8.03 & 84.16 +- 11.98 & 88.74 +- 6.68 & 89.42 +- 5.89 & 28.55 +- 13.85 \\
eurosat & 89.63 +- 3.65 & 88.40 +- 5.81 & 89.35 +- 3.42 & 88.28 +- 4.44 & 80.10 +- 4.71 & 86.00 +- 5.06 & 89.14 +- 3.59 & 88.42 +- 4.61 & 88.95 +- 4.28 & 76.10 +- 2.74 \\
fashion & 73.15 +- 19.71 & 75.59 +- 8.09 & 73.78 +- 17.00 & 75.15 +- 9.59 & 75.99 +- 10.19 & 64.91 +- 19.97 & 70.47 +- 19.44 & 75.26 +- 9.20 & 74.28 +- 8.98 & 78.47 +- 10.96 \\
food & 66.90 +- 1.29 & 66.80 +- 0.91 & 66.44 +- 1.17 & 66.58 +- 1.06 & 65.65 +- 0.93 & 64.83 +- 0.93 & 66.72 +- 1.19 & 65.91 +- 0.52 & 66.82 +- 0.91 & 58.91 +- 2.03 \\
garbage & 75.21 +- 6.08 & 78.70 +- 6.44 & 76.71 +- 5.98 & 70.57 +- 4.30 & 75.09 +- 5.66 & 70.46 +- 9.19 & 74.72 +- 5.51 & 80.19 +- 3.72 & 80.29 +- 3.96 & 72.67 +- 6.90 \\
rock & 64.04 +- 3.02 & 62.76 +- 2.69 & 64.44 +- 2.99 & 63.92 +- 5.48 & 67.57 +- 2.03 & 63.48 +- 3.28 & 64.17 +- 2.83 & 65.72 +- 2.95 & 64.68 +- 2.63 & 60.35 +- 17.22 \\
tissue & 65.49 +- 7.25 & 58.04 +- 11.71 & 64.26 +- 7.50 & 59.18 +- 10.23 & 65.22 +- 6.66 & 68.26 +- 3.35 & 65.88 +- 7.44 & 63.12 +- 11.74 & 55.59 +- 15.43 & 59.87 +- 8.35 \\
yoga & 72.78 +- 6.06 & 70.40 +- 7.99 & 72.86 +- 4.47 & 64.30 +- 7.67 & 69.35 +- 7.90 & 72.76 +- 5.29 & 71.81 +- 8.08 & 71.78 +- 6.40 & 68.92 +- 10.56 & 62.76 +- 6.75 \\
\bottomrule
\end{tabular}
}
\end{table*}

\begin{table*}[htbp]
  \centering
  \caption{AUROC (resnet backbone). Per-dataset out-of-domain OOD results.}
  \label{tab:supp:extended:auroc:resnet:out}
  \small
  \resizebox{\linewidth}{!}{\begin{tabular}{lllllllllll}
\toprule
method_display & Energy-Based (EBO) & Mahalanobis (MDS) & MaxSoftmax (MSP) & Nearest Neigh. (kNN) & Neural Collapse (NCI) & ReAct & SCALE & Virtual-Logit (ViM) & whitening & mds_teacher \\
ID Dataset &  &  &  &  &  &  &  &  &  &  \\
\midrule
colon & 86.74 +- 2.44 & 96.25 +- 4.28 & 90.91 +- 1.58 & 96.80 +- 3.61 & 95.14 +- 2.46 & 78.59 +- 10.29 & 86.51 +- 2.24 & 97.21 +- 2.80 & 95.07 +- 5.73 & 99.87 +- 0.04 \\
eurosat & 89.59 +- 2.44 & 97.31 +- 1.05 & 89.10 +- 2.64 & 92.87 +- 2.38 & 84.16 +- 3.62 & 87.86 +- 3.32 & 89.06 +- 2.68 & 96.72 +- 1.25 & 96.75 +- 1.14 & 99.95 +- 0.01 \\
fashion & 96.05 +- 1.60 & 98.49 +- 0.58 & 93.42 +- 2.33 & 96.26 +- 1.55 & 90.63 +- 2.16 & 90.57 +- 6.56 & 94.78 +- 1.64 & 97.87 +- 0.91 & 98.31 +- 0.70 & 99.75 +- 0.16 \\
food & 75.51 +- 2.09 & 84.18 +- 2.31 & 74.10 +- 1.76 & 74.15 +- 1.71 & 73.34 +- 1.54 & 74.79 +- 1.95 & 75.39 +- 1.77 & 82.59 +- 1.20 & 84.30 +- 2.22 & 99.74 +- 0.03 \\
garbage & 87.84 +- 15.00 & 94.72 +- 8.82 & 87.30 +- 10.27 & 90.17 +- 4.26 & 90.17 +- 9.35 & 87.70 +- 12.85 & 86.95 +- 14.29 & 97.29 +- 3.13 & 96.12 +- 6.01 & 99.98 +- 0.03 \\
rock & 68.00 +- 4.58 & 84.59 +- 2.97 & 65.93 +- 2.98 & 73.85 +- 4.95 & 71.60 +- 1.03 & 65.22 +- 3.86 & 67.96 +- 4.33 & 80.24 +- 3.90 & 84.99 +- 3.08 & 97.74 +- 0.17 \\
tissue & 71.45 +- 8.53 & 99.19 +- 0.23 & 71.26 +- 8.03 & 89.23 +- 2.27 & 84.78 +- 7.27 & 79.53 +- 12.13 & 72.37 +- 9.27 & 98.94 +- 0.64 & 98.87 +- 0.20 & 100.00 +- 0.00 \\
yoga & 87.70 +- 4.53 & 85.85 +- 21.40 & 86.51 +- 4.46 & 86.20 +- 13.98 & 88.08 +- 4.97 & 87.63 +- 7.54 & 87.17 +- 6.37 & 88.21 +- 16.05 & 80.34 +- 28.15 & 99.29 +- 0.26 \\
\bottomrule
\end{tabular}
}
\end{table*}

\begin{table*}[htbp]
  \centering
  \caption{AUROC (simclr backbone). Per-dataset in-domain OOD results.}
  \label{tab:supp:extended:auroc:simclr:in}
  \small
  \resizebox{\linewidth}{!}{\begin{tabular}{lll}
\toprule
method_display & Mahalanobis (MDS) & Nearest Neigh. (kNN) \\
ID Dataset &  &  \\
\midrule
colon & 77.48 +- 10.66 & 93.41 +- 3.26 \\
eurosat & 60.65 +- 16.45 & 92.03 +- 3.08 \\
fashion & 65.90 +- 7.75 & 83.19 +- 7.75 \\
food & 47.97 +- 2.16 & 72.87 +- 1.95 \\
garbage & 58.82 +- 7.91 & 58.89 +- 17.24 \\
rock & 51.05 +- 7.92 & 50.73 +- 9.37 \\
tissue & 56.07 +- 18.17 & 56.13 +- 12.05 \\
yoga & 60.84 +- 5.15 & 74.25 +- 10.84 \\
\bottomrule
\end{tabular}
}
\end{table*}

\begin{table*}[htbp]
  \centering
  \caption{AUROC (simclr backbone). Per-dataset out-of-domain OOD results.}
  \label{tab:supp:extended:auroc:simclr:out}
  \small
  \resizebox{\linewidth}{!}{\begin{tabular}{lll}
\toprule
method_display & Mahalanobis (MDS) & Nearest Neigh. (kNN) \\
ID Dataset &  &  \\
\midrule
colon & 99.28 +- 0.37 & 93.71 +- 4.47 \\
eurosat & 96.88 +- 1.16 & 94.73 +- 1.30 \\
fashion & 97.65 +- 1.02 & 92.27 +- 1.67 \\
food & 59.35 +- 3.69 & 83.36 +- 2.07 \\
garbage & 98.48 +- 0.58 & 71.62 +- 10.34 \\
rock & 66.69 +- 3.55 & 60.24 +- 5.38 \\
tissue & 99.99 +- 0.01 & 80.44 +- 7.17 \\
yoga & 63.14 +- 5.55 & 85.06 +- 13.00 \\
\bottomrule
\end{tabular}
}
\end{table*}

\begin{table*}[htbp]
  \centering
  \caption{AUROC (tra\_dino backbone). Per-dataset in-domain OOD results.}
  \label{tab:supp:extended:auroc:tra_dino:in}
  \small
  \resizebox{\linewidth}{!}{\begin{tabular}{llllllllll}
\toprule
method_display & Energy-Based (EBO) & Mahalanobis (MDS) & MaxSoftmax (MSP) & Nearest Neigh. (kNN) & Neural Collapse (NCI) & ReAct & SCALE & Virtual-Logit (ViM) & whitening \\
ID Dataset &  &  &  &  &  &  &  &  &  \\
\midrule
colon & 91.50 +- 5.35 & 94.21 +- 3.15 & 92.56 +- 4.26 & 94.50 +- 2.78 & 91.66 +- 3.98 & 91.35 +- 5.15 & 91.55 +- 4.00 & 94.90 +- 2.82 & 94.25 +- 3.14 \\
eurosat & 93.29 +- 1.76 & 90.67 +- 3.32 & 91.15 +- 2.98 & 91.20 +- 3.90 & 92.49 +- 2.28 & 93.29 +- 1.87 & 92.94 +- 2.00 & 93.69 +- 1.75 & 90.83 +- 3.25 \\
fashion & 80.67 +- 12.75 & 78.96 +- 9.66 & 78.90 +- 11.77 & 79.48 +- 11.13 & 81.07 +- 11.99 & 80.53 +- 13.19 & 77.39 +- 13.19 & 81.12 +- 11.03 & 79.06 +- 9.71 \\
food & 72.61 +- 0.82 & 67.80 +- 1.34 & 70.48 +- 0.91 & 69.50 +- 1.47 & 72.59 +- 0.83 & 72.63 +- 0.74 & 71.21 +- 1.19 & 71.44 +- 1.10 & 69.23 +- 1.30 \\
garbage & 73.32 +- 6.58 & 78.27 +- 3.17 & 71.86 +- 8.12 & 68.09 +- 5.86 & 69.57 +- 7.01 & 71.89 +- 7.80 & 73.61 +- 7.17 & 78.57 +- 5.02 & 78.29 +- 3.18 \\
rock & 64.11 +- 3.70 & 57.02 +- 11.60 & 59.34 +- 4.17 & 55.24 +- 9.06 & 59.93 +- 4.99 & 65.83 +- 3.84 & 64.45 +- 4.51 & 61.68 +- 7.70 & 56.94 +- 11.66 \\
tissue & 62.94 +- 13.35 & 62.18 +- 9.14 & 61.05 +- 8.92 & 60.59 +- 10.54 & 52.78 +- 12.65 & 62.51 +- 12.75 & 62.74 +- 11.54 & 63.97 +- 13.41 & 61.84 +- 8.75 \\
yoga & 67.92 +- 9.12 & 62.62 +- 4.86 & 67.14 +- 7.10 & 59.96 +- 9.24 & 62.54 +- 7.94 & 68.28 +- 9.04 & 64.75 +- 11.14 & 67.53 +- 7.89 & 62.66 +- 4.91 \\
\bottomrule
\end{tabular}
}
\end{table*}

\begin{table*}[htbp]
  \centering
  \caption{AUROC (tra\_dino backbone). Per-dataset out-of-domain OOD results.}
  \label{tab:supp:extended:auroc:tra_dino:out}
  \small
  \resizebox{\linewidth}{!}{\begin{tabular}{llllllllll}
\toprule
method_display & Energy-Based (EBO) & Mahalanobis (MDS) & MaxSoftmax (MSP) & Nearest Neigh. (kNN) & Neural Collapse (NCI) & ReAct & SCALE & Virtual-Logit (ViM) & whitening \\
ID Dataset &  &  &  &  &  &  &  &  &  \\
\midrule
colon & 98.56 +- 1.44 & 99.99 +- 0.02 & 98.42 +- 1.74 & 99.73 +- 0.23 & 99.20 +- 0.70 & 98.74 +- 1.43 & 97.50 +- 2.60 & 99.98 +- 0.02 & 99.99 +- 0.03 \\
eurosat & 96.02 +- 2.88 & 98.27 +- 1.48 & 94.83 +- 3.23 & 97.41 +- 2.09 & 95.69 +- 3.91 & 96.20 +- 2.91 & 95.91 +- 2.50 & 98.42 +- 1.30 & 98.28 +- 1.47 \\
fashion & 96.71 +- 1.90 & 98.69 +- 0.43 & 92.41 +- 4.25 & 97.16 +- 1.16 & 96.53 +- 2.31 & 96.79 +- 1.91 & 97.03 +- 1.03 & 98.63 +- 0.59 & 98.70 +- 0.43 \\
food & 87.40 +- 1.03 & 86.36 +- 3.45 & 79.44 +- 0.70 & 81.13 +- 3.29 & 87.11 +- 0.98 & 87.80 +- 1.09 & 85.44 +- 2.92 & 87.46 +- 1.75 & 87.36 +- 3.10 \\
garbage & 83.60 +- 2.81 & 92.69 +- 3.13 & 78.41 +- 3.22 & 83.39 +- 3.72 & 82.69 +- 4.00 & 83.95 +- 2.75 & 84.62 +- 3.12 & 93.58 +- 1.93 & 92.68 +- 3.13 \\
rock & 56.12 +- 7.26 & 78.26 +- 3.98 & 56.31 +- 10.15 & 70.89 +- 5.44 & 61.14 +- 6.28 & 58.61 +- 6.22 & 55.56 +- 7.92 & 76.80 +- 3.22 & 78.18 +- 4.10 \\
tissue & 77.87 +- 17.84 & 99.90 +- 0.08 & 72.71 +- 16.89 & 98.81 +- 0.57 & 67.44 +- 4.32 & 79.69 +- 14.38 & 82.49 +- 13.32 & 99.55 +- 0.28 & 99.90 +- 0.08 \\
yoga & 73.19 +- 8.69 & 81.77 +- 9.93 & 71.31 +- 8.49 & 78.06 +- 7.01 & 73.17 +- 6.75 & 73.56 +- 8.52 & 74.99 +- 5.77 & 80.81 +- 7.80 & 81.79 +- 9.91 \\
\bottomrule
\end{tabular}
}
\end{table*}

\begin{table*}[htbp]
  \centering
  \caption{AUROC (tra\_resnet backbone). Per-dataset in-domain OOD results.}
  \label{tab:supp:extended:auroc:tra_resnet:in}
  \small
  \resizebox{\linewidth}{!}{\begin{tabular}{llllllllll}
\toprule
method_display & Energy-Based (EBO) & Mahalanobis (MDS) & MaxSoftmax (MSP) & Nearest Neigh. (kNN) & Neural Collapse (NCI) & ReAct & SCALE & Virtual-Logit (ViM) & whitening \\
ID Dataset &  &  &  &  &  &  &  &  &  \\
\midrule
colon & 89.49 +- 7.51 & 94.75 +- 3.21 & 89.60 +- 7.34 & 93.76 +- 3.51 & 81.58 +- 7.48 & 83.34 +- 4.78 & 89.11 +- 7.54 & 93.42 +- 3.51 & 94.44 +- 3.52 \\
eurosat & 93.15 +- 1.59 & 89.34 +- 6.05 & 92.92 +- 1.02 & 90.86 +- 1.63 & 91.19 +- 6.15 & 89.50 +- 6.59 & 92.87 +- 1.90 & 94.54 +- 1.10 & 89.41 +- 5.95 \\
fashion & 78.12 +- 10.48 & 80.22 +- 7.03 & 78.37 +- 9.55 & 79.95 +- 9.67 & 77.12 +- 7.37 & 74.59 +- 16.40 & 79.54 +- 12.33 & 81.00 +- 6.91 & 80.26 +- 7.17 \\
food & 69.74 +- 1.27 & 68.08 +- 0.66 & 69.23 +- 1.16 & 69.98 +- 1.10 & 69.77 +- 1.11 & 65.41 +- 2.56 & 67.47 +- 2.84 & 69.67 +- 0.59 & 68.37 +- 0.68 \\
garbage & 78.28 +- 7.73 & 81.53 +- 7.71 & 78.39 +- 5.99 & 67.28 +- 7.31 & 75.52 +- 5.79 & 75.08 +- 10.08 & 78.43 +- 7.23 & 80.69 +- 6.49 & 83.43 +- 5.58 \\
rock & 68.36 +- 5.17 & 63.28 +- 4.84 & 66.63 +- 7.47 & 63.80 +- 5.18 & 71.72 +- 2.77 & 68.63 +- 3.15 & 68.50 +- 5.22 & 69.54 +- 2.92 & 65.63 +- 4.69 \\
tissue & 64.15 +- 10.38 & 55.46 +- 11.62 & 63.61 +- 10.00 & 56.95 +- 8.87 & 66.25 +- 6.49 & 66.70 +- 5.33 & 64.25 +- 7.88 & 64.50 +- 12.45 & 55.60 +- 11.54 \\
yoga & 74.52 +- 7.80 & 70.75 +- 12.05 & 74.17 +- 6.62 & 60.86 +- 8.24 & 65.53 +- 7.97 & 72.36 +- 7.15 & 73.89 +- 8.32 & 70.86 +- 9.13 & 71.53 +- 11.40 \\
\bottomrule
\end{tabular}
}
\end{table*}

\begin{table*}[htbp]
  \centering
  \caption{AUROC (tra\_resnet backbone). Per-dataset out-of-domain OOD results.}
  \label{tab:supp:extended:auroc:tra_resnet:out}
  \small
  \resizebox{\linewidth}{!}{\begin{tabular}{llllllllll}
\toprule
method_display & Energy-Based (EBO) & Mahalanobis (MDS) & MaxSoftmax (MSP) & Nearest Neigh. (kNN) & Neural Collapse (NCI) & ReAct & SCALE & Virtual-Logit (ViM) & whitening \\
ID Dataset &  &  &  &  &  &  &  &  &  \\
\midrule
colon & 94.87 +- 4.37 & 99.80 +- 0.21 & 95.38 +- 4.27 & 99.30 +- 0.40 & 98.67 +- 1.53 & 93.06 +- 3.12 & 94.83 +- 4.05 & 99.75 +- 0.29 & 99.79 +- 0.24 \\
eurosat & 97.60 +- 0.53 & 99.57 +- 0.21 & 96.92 +- 0.71 & 98.76 +- 0.47 & 97.27 +- 2.12 & 93.91 +- 4.92 & 97.59 +- 0.55 & 99.50 +- 0.23 & 99.56 +- 0.22 \\
fashion & 84.77 +- 9.40 & 99.32 +- 0.57 & 84.49 +- 7.91 & 96.92 +- 2.54 & 96.70 +- 1.74 & 91.83 +- 6.49 & 83.39 +- 10.47 & 99.41 +- 0.50 & 99.36 +- 0.55 \\
food & 81.43 +- 2.71 & 89.14 +- 1.80 & 78.37 +- 2.28 & 84.68 +- 2.51 & 82.92 +- 2.30 & 86.31 +- 3.05 & 84.06 +- 2.01 & 91.15 +- 1.44 & 89.31 +- 1.76 \\
garbage & 87.83 +- 16.83 & 94.82 +- 9.88 & 88.15 +- 10.46 & 90.45 +- 8.77 & 89.58 +- 7.35 & 89.94 +- 9.60 & 87.47 +- 17.29 & 97.20 +- 1.94 & 96.74 +- 5.93 \\
rock & 73.48 +- 6.94 & 90.00 +- 2.63 & 69.78 +- 5.80 & 85.96 +- 2.62 & 75.98 +- 6.08 & 72.35 +- 7.57 & 73.42 +- 6.81 & 84.55 +- 3.98 & 90.10 +- 2.57 \\
tissue & 75.42 +- 7.83 & 99.76 +- 0.18 & 74.51 +- 6.86 & 94.74 +- 1.92 & 91.06 +- 2.21 & 80.53 +- 14.66 & 77.79 +- 9.37 & 99.88 +- 0.13 & 99.77 +- 0.18 \\
yoga & 88.07 +- 4.51 & 87.89 +- 13.35 & 86.96 +- 4.84 & 88.44 +- 9.84 & 87.37 +- 7.04 & 88.50 +- 5.64 & 87.34 +- 5.46 & 91.10 +- 6.09 & 87.50 +- 14.59 \\
\bottomrule
\end{tabular}
}
\end{table*}

\subsection{AUPR-IN}

\begin{table*}[htbp]
  \centering
  \caption{AUPR\_IN (dino backbone). Per-dataset in-domain OOD results.}
  \label{tab:supp:extended:auprin:dino:in}
  \small
  \resizebox{\linewidth}{!}{\begin{tabular}{llllllllll}
\toprule
method_display & Energy-Based (EBO) & Mahalanobis (MDS) & MaxSoftmax (MSP) & Nearest Neigh. (kNN) & Neural Collapse (NCI) & ReAct & SCALE & Virtual-Logit (ViM) & whitening \\
ID Dataset &  &  &  &  &  &  &  &  &  \\
\midrule
colon & 93.50 +- 4.94 & 95.02 +- 3.34 & 91.28 +- 7.23 & 95.68 +- 3.03 & 95.75 +- 1.90 & 93.57 +- 4.96 & 95.54 +- 1.56 & 95.85 +- 2.66 & 94.72 +- 3.56 \\
eurosat & 91.04 +- 1.55 & 95.61 +- 1.72 & 90.44 +- 2.68 & 96.15 +- 1.31 & 89.70 +- 4.20 & 90.29 +- 2.22 & 89.18 +- 4.35 & 95.81 +- 1.69 & 95.59 +- 1.73 \\
fashion & 85.06 +- 5.89 & 89.29 +- 3.12 & 87.10 +- 6.00 & 89.24 +- 4.50 & 85.41 +- 6.40 & 85.88 +- 4.86 & 87.71 +- 5.66 & 88.45 +- 4.23 & 89.39 +- 3.05 \\
food & 84.55 +- 0.46 & 81.08 +- 0.49 & 83.39 +- 0.56 & 82.44 +- 0.71 & 84.59 +- 0.53 & 84.54 +- 0.43 & 83.54 +- 0.54 & 84.22 +- 0.48 & 82.40 +- 0.44 \\
garbage & 79.99 +- 6.61 & 81.99 +- 4.44 & 79.46 +- 5.95 & 75.55 +- 9.09 & 78.37 +- 6.14 & 79.68 +- 7.05 & 79.84 +- 6.10 & 83.12 +- 4.52 & 82.01 +- 4.44 \\
rock & 82.98 +- 7.42 & 81.23 +- 3.52 & 81.79 +- 8.83 & 82.58 +- 5.02 & 82.36 +- 5.90 & 83.62 +- 7.05 & 84.49 +- 7.77 & 84.63 +- 4.30 & 81.29 +- 3.41 \\
tissue & 71.61 +- 27.67 & 69.15 +- 25.28 & 70.42 +- 27.57 & 68.00 +- 25.60 & 63.05 +- 27.05 & 71.96 +- 27.55 & 71.27 +- 27.63 & 71.74 +- 27.87 & 68.89 +- 24.90 \\
yoga & 81.38 +- 4.28 & 75.94 +- 9.47 & 81.38 +- 3.81 & 80.75 +- 3.35 & 74.33 +- 6.91 & 77.00 +- 1.42 & 78.00 +- 5.06 & 81.15 +- 8.02 & 77.76 +- 9.08 \\
\bottomrule
\end{tabular}
}
\end{table*}

\begin{table*}[htbp]
  \centering
  \caption{AUPR\_IN (dino backbone). Per-dataset out-of-domain OOD results.}
  \label{tab:supp:extended:auprin:dino:out}
  \small
  \resizebox{\linewidth}{!}{\begin{tabular}{llllllllll}
\toprule
method_display & Energy-Based (EBO) & Mahalanobis (MDS) & MaxSoftmax (MSP) & Nearest Neigh. (kNN) & Neural Collapse (NCI) & ReAct & SCALE & Virtual-Logit (ViM) & whitening \\
ID Dataset &  &  &  &  &  &  &  &  &  \\
\midrule
colon & 94.39 +- 6.79 & 99.98 +- 0.02 & 94.96 +- 5.97 & 99.47 +- 0.54 & 95.63 +- 7.00 & 95.28 +- 6.50 & 90.39 +- 8.40 & 99.98 +- 0.02 & 99.96 +- 0.03 \\
eurosat & 87.45 +- 3.87 & 97.11 +- 1.52 & 86.12 +- 4.78 & 85.13 +- 10.69 & 89.59 +- 8.48 & 88.54 +- 3.29 & 86.38 +- 3.52 & 97.61 +- 1.33 & 97.05 +- 1.52 \\
fashion & 73.60 +- 9.91 & 96.64 +- 1.90 & 69.13 +- 6.64 & 95.10 +- 2.24 & 82.38 +- 7.09 & 77.57 +- 8.91 & 77.86 +- 10.29 & 96.02 +- 2.29 & 96.60 +- 1.94 \\
food & 83.82 +- 1.04 & 86.16 +- 1.97 & 75.86 +- 1.04 & 79.61 +- 4.12 & 83.08 +- 1.29 & 84.51 +- 0.91 & 80.50 +- 1.88 & 85.75 +- 0.87 & 86.81 +- 1.58 \\
garbage & 32.96 +- 9.56 & 59.61 +- 6.79 & 28.54 +- 4.60 & 41.13 +- 4.48 & 34.72 +- 5.25 & 33.76 +- 8.73 & 32.41 +- 12.82 & 58.06 +- 4.17 & 59.52 +- 6.77 \\
rock & 9.30 +- 3.12 & 26.00 +- 5.47 & 9.28 +- 2.27 & 18.51 +- 1.92 & 15.13 +- 2.33 & 10.47 +- 3.89 & 9.81 +- 4.32 & 21.50 +- 3.65 & 26.08 +- 5.69 \\
tissue & 68.32 +- 22.60 & 99.88 +- 0.19 & 61.15 +- 18.73 & 99.43 +- 0.48 & 71.09 +- 9.97 & 70.29 +- 21.05 & 66.80 +- 22.42 & 99.79 +- 0.25 & 99.85 +- 0.20 \\
yoga & 25.98 +- 11.80 & 45.49 +- 26.18 & 20.59 +- 7.98 & 41.22 +- 11.29 & 29.31 +- 10.13 & 27.98 +- 8.33 & 28.10 +- 14.47 & 40.13 +- 20.99 & 45.61 +- 22.67 \\
\bottomrule
\end{tabular}
}
\end{table*}

\begin{table*}[htbp]
  \centering
  \caption{AUPR\_IN (resnet backbone). Per-dataset in-domain OOD results.}
  \label{tab:supp:extended:auprin:resnet:in}
  \small
  \resizebox{\linewidth}{!}{\begin{tabular}{lllllllllll}
\toprule
method_display & Energy-Based (EBO) & Mahalanobis (MDS) & MaxSoftmax (MSP) & Nearest Neigh. (kNN) & Neural Collapse (NCI) & ReAct & SCALE & Virtual-Logit (ViM) & whitening & mds_teacher \\
ID Dataset &  &  &  &  &  &  &  &  &  &  \\
\midrule
colon & 86.42 +- 12.85 & 94.07 +- 3.91 & 86.68 +- 12.83 & 93.65 +- 5.67 & 79.10 +- 15.77 & 92.18 +- 5.26 & 86.16 +- 11.97 & 92.92 +- 5.75 & 93.82 +- 4.50 & 52.93 +- 7.78 \\
eurosat & 94.58 +- 2.06 & 93.52 +- 3.84 & 94.54 +- 1.97 & 93.75 +- 3.32 & 87.78 +- 3.45 & 92.35 +- 3.29 & 94.36 +- 2.03 & 93.77 +- 3.10 & 94.05 +- 2.92 & 88.17 +- 2.19 \\
fashion & 83.74 +- 13.18 & 86.27 +- 5.03 & 84.14 +- 12.28 & 85.24 +- 5.93 & 87.62 +- 5.77 & 78.38 +- 12.29 & 82.23 +- 13.66 & 85.61 +- 5.55 & 85.08 +- 5.50 & 88.99 +- 5.74 \\
food & 80.76 +- 1.21 & 80.88 +- 0.86 & 80.64 +- 1.12 & 81.31 +- 0.88 & 79.02 +- 0.82 & 77.39 +- 0.98 & 80.48 +- 1.14 & 80.63 +- 0.70 & 80.87 +- 0.87 & 73.52 +- 1.27 \\
garbage & 83.00 +- 3.02 & 87.68 +- 2.64 & 84.45 +- 3.03 & 82.43 +- 1.36 & 84.05 +- 1.90 & 80.22 +- 5.74 & 82.31 +- 3.02 & 88.06 +- 1.97 & 88.48 +- 2.03 & 80.38 +- 7.65 \\
rock & 83.61 +- 4.40 & 83.29 +- 6.54 & 84.25 +- 4.12 & 85.68 +- 5.65 & 86.45 +- 4.29 & 83.09 +- 5.68 & 83.68 +- 4.33 & 85.68 +- 5.01 & 84.87 +- 5.93 & 84.80 +- 8.99 \\
tissue & 73.50 +- 22.45 & 67.37 +- 26.72 & 73.10 +- 22.79 & 67.16 +- 26.30 & 73.74 +- 18.08 & 76.81 +- 13.50 & 73.60 +- 21.88 & 70.85 +- 26.92 & 52.94 +- 25.77 & 68.97 +- 25.40 \\
yoga & 84.21 +- 5.70 & 83.31 +- 8.33 & 84.55 +- 4.55 & 77.96 +- 6.83 & 81.01 +- 7.45 & 84.24 +- 5.73 & 83.57 +- 7.07 & 83.73 +- 6.81 & 81.46 +- 11.26 & 78.71 +- 7.45 \\
\bottomrule
\end{tabular}
}
\end{table*}

\begin{table*}[htbp]
  \centering
  \caption{AUPR\_IN (resnet backbone). Per-dataset out-of-domain OOD results.}
  \label{tab:supp:extended:auprin:resnet:out}
  \small
  \resizebox{\linewidth}{!}{\begin{tabular}{lllllllllll}
\toprule
method_display & Energy-Based (EBO) & Mahalanobis (MDS) & MaxSoftmax (MSP) & Nearest Neigh. (kNN) & Neural Collapse (NCI) & ReAct & SCALE & Virtual-Logit (ViM) & whitening & mds_teacher \\
ID Dataset &  &  &  &  &  &  &  &  &  &  \\
\midrule
colon & 68.47 +- 3.94 & 93.87 +- 6.63 & 76.05 +- 3.94 & 93.71 +- 7.40 & 91.51 +- 3.57 & 70.42 +- 12.73 & 68.10 +- 3.68 & 95.24 +- 4.66 & 91.20 +- 10.09 & 99.86 +- 0.05 \\
eurosat & 61.75 +- 5.50 & 90.70 +- 4.21 & 61.64 +- 5.82 & 79.27 +- 6.60 & 55.39 +- 7.93 & 58.67 +- 7.87 & 61.36 +- 5.47 & 88.28 +- 4.79 & 88.39 +- 4.22 & 99.84 +- 0.03 \\
fashion & 92.31 +- 3.06 & 96.68 +- 0.88 & 89.59 +- 4.06 & 92.67 +- 3.00 & 80.55 +- 4.61 & 81.75 +- 12.42 & 90.04 +- 3.20 & 95.91 +- 1.49 & 96.28 +- 1.17 & 99.04 +- 0.57 \\
food & 68.77 +- 1.89 & 77.65 +- 2.32 & 67.82 +- 1.55 & 69.28 +- 1.59 & 66.76 +- 1.59 & 64.52 +- 1.48 & 68.56 +- 1.58 & 76.57 +- 1.53 & 77.69 +- 2.25 & 99.67 +- 0.03 \\
garbage & 51.82 +- 25.89 & 79.07 +- 20.33 & 52.54 +- 24.05 & 57.96 +- 8.65 & 59.16 +- 20.39 & 50.58 +- 20.58 & 46.27 +- 24.97 & 81.06 +- 17.13 & 81.24 +- 18.24 & 99.81 +- 0.11 \\
rock & 11.77 +- 3.62 & 34.57 +- 6.46 & 12.20 +- 3.58 & 24.18 +- 5.14 & 18.51 +- 1.38 & 8.43 +- 2.27 & 11.62 +- 3.33 & 29.98 +- 1.93 & 36.14 +- 6.68 & 87.83 +- 3.15 \\
tissue & 65.59 +- 16.53 & 98.50 +- 0.94 & 66.53 +- 17.08 & 88.43 +- 3.92 & 84.73 +- 13.59 & 78.74 +- 19.29 & 67.99 +- 17.94 & 97.81 +- 2.31 & 97.64 +- 0.72 & 100.00 +- 0.00 \\
yoga & 49.76 +- 11.28 & 63.26 +- 22.91 & 49.36 +- 11.14 & 58.62 +- 24.15 & 54.20 +- 13.02 & 47.47 +- 14.86 & 46.53 +- 10.37 & 66.71 +- 21.47 & 57.21 +- 30.68 & 97.78 +- 0.51 \\
\bottomrule
\end{tabular}
}
\end{table*}

\begin{table*}[htbp]
  \centering
  \caption{AUPR\_IN (simclr backbone). Per-dataset in-domain OOD results.}
  \label{tab:supp:extended:auprin:simclr:in}
  \small
  \resizebox{\linewidth}{!}{\begin{tabular}{lll}
\toprule
method_display & Mahalanobis (MDS) & Nearest Neigh. (kNN) \\
ID Dataset &  &  \\
\midrule
colon & 86.16 +- 6.69 & 96.50 +- 1.80 \\
eurosat & 79.29 +- 10.01 & 95.81 +- 1.44 \\
fashion & 82.45 +- 5.17 & 91.33 +- 3.79 \\
food & 65.31 +- 1.57 & 85.46 +- 1.36 \\
garbage & 69.44 +- 10.63 & 72.48 +- 9.91 \\
rock & 79.35 +- 4.29 & 78.16 +- 5.47 \\
tissue & 66.83 +- 29.88 & 67.22 +- 27.53 \\
yoga & 79.27 +- 1.96 & 83.75 +- 10.22 \\
\bottomrule
\end{tabular}
}
\end{table*}

\begin{table*}[htbp]
  \centering
  \caption{AUPR\_IN (simclr backbone). Per-dataset out-of-domain OOD results.}
  \label{tab:supp:extended:auprin:simclr:out}
  \small
  \resizebox{\linewidth}{!}{\begin{tabular}{lll}
\toprule
method_display & Mahalanobis (MDS) & Nearest Neigh. (kNN) \\
ID Dataset &  &  \\
\midrule
colon & 99.04 +- 0.50 & 90.39 +- 5.74 \\
eurosat & 91.46 +- 2.60 & 81.85 +- 4.38 \\
fashion & 93.49 +- 2.09 & 87.55 +- 2.40 \\
food & 47.82 +- 4.01 & 80.98 +- 1.64 \\
garbage & 85.93 +- 4.28 & 18.68 +- 5.97 \\
rock & 17.61 +- 4.56 & 5.93 +- 2.73 \\
tissue & 99.97 +- 0.05 & 78.90 +- 5.36 \\
yoga & 19.97 +- 7.52 & 49.51 +- 23.15 \\
\bottomrule
\end{tabular}
}
\end{table*}

\begin{table*}[htbp]
  \centering
  \caption{AUPR\_IN (tra\_dino backbone). Per-dataset in-domain OOD results.}
  \label{tab:supp:extended:auprin:tra_dino:in}
  \small
  \resizebox{\linewidth}{!}{\begin{tabular}{llllllllll}
\toprule
method_display & Energy-Based (EBO) & Mahalanobis (MDS) & MaxSoftmax (MSP) & Nearest Neigh. (kNN) & Neural Collapse (NCI) & ReAct & SCALE & Virtual-Logit (ViM) & whitening \\
ID Dataset &  &  &  &  &  &  &  &  &  \\
\midrule
colon & 93.17 +- 5.82 & 96.85 +- 2.36 & 94.30 +- 4.92 & 97.16 +- 1.66 & 93.81 +- 4.42 & 93.16 +- 5.60 & 94.01 +- 3.87 & 96.91 +- 2.44 & 96.87 +- 2.35 \\
eurosat & 96.49 +- 1.32 & 95.60 +- 1.59 & 95.39 +- 1.98 & 95.59 +- 1.98 & 95.96 +- 1.77 & 96.52 +- 1.30 & 96.26 +- 1.29 & 97.06 +- 0.88 & 95.69 +- 1.55 \\
fashion & 89.98 +- 7.60 & 89.62 +- 6.15 & 89.21 +- 7.60 & 89.95 +- 6.56 & 90.24 +- 6.91 & 89.83 +- 8.05 & 86.29 +- 9.68 & 90.35 +- 7.31 & 89.68 +- 6.20 \\
food & 85.43 +- 0.53 & 82.04 +- 0.88 & 84.29 +- 0.67 & 83.11 +- 0.95 & 85.50 +- 0.57 & 85.47 +- 0.52 & 84.09 +- 0.75 & 84.92 +- 0.66 & 83.30 +- 0.84 \\
garbage & 81.25 +- 6.13 & 86.17 +- 4.15 & 81.33 +- 5.33 & 79.53 +- 6.48 & 79.55 +- 4.20 & 79.80 +- 7.02 & 81.40 +- 5.72 & 86.53 +- 3.78 & 86.19 +- 4.15 \\
rock & 84.65 +- 4.19 & 79.76 +- 9.05 & 81.98 +- 4.70 & 79.00 +- 8.22 & 82.10 +- 9.48 & 85.51 +- 3.44 & 84.46 +- 4.08 & 83.50 +- 5.33 & 79.75 +- 9.15 \\
tissue & 72.47 +- 27.63 & 70.56 +- 26.25 & 71.19 +- 25.80 & 70.48 +- 27.23 & 64.79 +- 27.36 & 72.18 +- 27.32 & 72.26 +- 26.95 & 73.03 +- 27.64 & 70.46 +- 26.11 \\
yoga & 80.20 +- 4.91 & 78.16 +- 2.80 & 81.19 +- 4.26 & 76.09 +- 8.24 & 76.97 +- 4.62 & 81.06 +- 4.96 & 77.59 +- 6.88 & 81.44 +- 4.49 & 78.18 +- 2.83 \\
\bottomrule
\end{tabular}
}
\end{table*}

\begin{table*}[htbp]
  \centering
  \caption{AUPR\_IN (tra\_dino backbone). Per-dataset out-of-domain OOD results.}
  \label{tab:supp:extended:auprin:tra_dino:out}
  \small
  \resizebox{\linewidth}{!}{\begin{tabular}{llllllllll}
\toprule
method_display & Energy-Based (EBO) & Mahalanobis (MDS) & MaxSoftmax (MSP) & Nearest Neigh. (kNN) & Neural Collapse (NCI) & ReAct & SCALE & Virtual-Logit (ViM) & whitening \\
ID Dataset &  &  &  &  &  &  &  &  &  \\
\midrule
colon & 97.96 +- 2.11 & 99.98 +- 0.03 & 98.00 +- 2.41 & 99.56 +- 0.49 & 98.84 +- 1.14 & 98.23 +- 2.10 & 96.36 +- 3.89 & 99.97 +- 0.03 & 99.98 +- 0.02 \\
eurosat & 88.03 +- 8.56 & 94.52 +- 4.67 & 87.12 +- 8.92 & 92.02 +- 6.33 & 87.00 +- 11.02 & 88.52 +- 8.76 & 87.56 +- 7.65 & 95.01 +- 3.93 & 94.54 +- 4.66 \\
fashion & 94.36 +- 2.65 & 97.39 +- 0.69 & 89.90 +- 5.06 & 94.74 +- 1.44 & 94.09 +- 3.38 & 94.45 +- 2.71 & 94.51 +- 1.51 & 97.44 +- 0.73 & 97.40 +- 0.68 \\
food & 85.34 +- 1.70 & 85.45 +- 2.93 & 77.25 +- 1.71 & 79.52 +- 3.37 & 85.27 +- 1.67 & 85.78 +- 1.87 & 83.28 +- 2.81 & 86.27 +- 1.75 & 86.45 +- 2.68 \\
garbage & 39.44 +- 8.95 & 67.53 +- 9.62 & 33.02 +- 8.86 & 42.73 +- 6.19 & 41.45 +- 4.25 & 40.53 +- 8.52 & 42.38 +- 11.22 & 69.97 +- 6.50 & 67.47 +- 9.62 \\
rock & 9.62 +- 3.62 & 28.83 +- 4.53 & 9.82 +- 2.67 & 20.58 +- 5.60 & 12.86 +- 5.86 & 10.75 +- 3.20 & 8.62 +- 4.30 & 24.09 +- 2.54 & 28.97 +- 4.75 \\
tissue & 76.71 +- 21.29 & 99.82 +- 0.19 & 70.95 +- 21.32 & 98.18 +- 0.59 & 66.30 +- 5.88 & 78.26 +- 18.20 & 81.45 +- 17.05 & 98.97 +- 0.88 & 99.81 +- 0.19 \\
yoga & 29.00 +- 17.93 & 50.73 +- 17.32 & 26.15 +- 15.76 & 40.41 +- 6.89 & 30.41 +- 14.33 & 29.65 +- 17.90 & 30.41 +- 12.75 & 45.52 +- 17.72 & 50.79 +- 17.28 \\
\bottomrule
\end{tabular}
}
\end{table*}

\begin{table*}[htbp]
  \centering
  \caption{AUPR\_IN (tra\_resnet backbone). Per-dataset in-domain OOD results.}
  \label{tab:supp:extended:auprin:tra_resnet:in}
  \small
  \resizebox{\linewidth}{!}{\begin{tabular}{llllllllll}
\toprule
method_display & Energy-Based (EBO) & Mahalanobis (MDS) & MaxSoftmax (MSP) & Nearest Neigh. (kNN) & Neural Collapse (NCI) & ReAct & SCALE & Virtual-Logit (ViM) & whitening \\
ID Dataset &  &  &  &  &  &  &  &  &  \\
\midrule
colon & 91.44 +- 7.28 & 96.84 +- 2.62 & 91.57 +- 7.16 & 96.31 +- 2.79 & 87.32 +- 5.38 & 90.00 +- 3.55 & 90.91 +- 7.57 & 96.02 +- 2.92 & 96.61 +- 2.94 \\
eurosat & 95.64 +- 1.60 & 94.53 +- 3.07 & 95.79 +- 1.16 & 95.15 +- 0.83 & 94.95 +- 3.80 & 94.17 +- 3.40 & 95.38 +- 1.93 & 97.27 +- 0.46 & 94.57 +- 3.03 \\
fashion & 86.74 +- 7.99 & 89.77 +- 4.11 & 87.31 +- 7.33 & 89.88 +- 5.52 & 87.83 +- 3.82 & 86.39 +- 7.90 & 87.28 +- 9.40 & 90.31 +- 3.82 & 89.81 +- 4.15 \\
food & 83.29 +- 0.72 & 82.02 +- 0.45 & 83.03 +- 0.69 & 83.47 +- 0.66 & 83.26 +- 0.62 & 78.53 +- 2.46 & 79.38 +- 4.06 & 83.25 +- 0.41 & 82.26 +- 0.46 \\
garbage & 85.27 +- 4.48 & 88.86 +- 4.63 & 85.87 +- 3.57 & 79.33 +- 6.25 & 84.22 +- 1.71 & 83.36 +- 6.62 & 85.34 +- 4.27 & 88.27 +- 3.70 & 90.19 +- 3.21 \\
rock & 86.71 +- 4.28 & 83.18 +- 8.93 & 86.08 +- 4.35 & 85.12 +- 5.87 & 88.48 +- 3.85 & 86.32 +- 3.64 & 86.77 +- 4.24 & 87.50 +- 5.11 & 84.74 +- 7.48 \\
tissue & 72.16 +- 26.02 & 65.44 +- 26.75 & 72.13 +- 25.91 & 65.61 +- 25.02 & 75.30 +- 17.80 & 76.45 +- 15.87 & 72.07 +- 21.75 & 71.31 +- 27.65 & 65.40 +- 26.74 \\
yoga & 84.96 +- 6.18 & 83.38 +- 10.69 & 84.93 +- 5.35 & 74.77 +- 8.82 & 76.40 +- 7.79 & 83.63 +- 6.16 & 84.22 +- 6.99 & 83.04 +- 7.80 & 83.82 +- 10.08 \\
\bottomrule
\end{tabular}
}
\end{table*}

\begin{table*}[htbp]
  \centering
  \caption{AUPR\_IN (tra\_resnet backbone). Per-dataset out-of-domain OOD results.}
  \label{tab:supp:extended:auprin:tra_resnet:out}
  \small
  \resizebox{\linewidth}{!}{\begin{tabular}{llllllllll}
\toprule
method_display & Energy-Based (EBO) & Mahalanobis (MDS) & MaxSoftmax (MSP) & Nearest Neigh. (kNN) & Neural Collapse (NCI) & ReAct & SCALE & Virtual-Logit (ViM) & whitening \\
ID Dataset &  &  &  &  &  &  &  &  &  \\
\midrule
colon & 85.70 +- 10.97 & 99.59 +- 0.44 & 87.64 +- 11.61 & 98.92 +- 0.66 & 97.63 +- 2.72 & 90.35 +- 2.16 & 85.13 +- 10.64 & 99.52 +- 0.59 & 99.56 +- 0.52 \\
eurosat & 90.00 +- 2.33 & 98.45 +- 0.69 & 89.31 +- 2.66 & 95.89 +- 1.36 & 90.08 +- 6.99 & 79.81 +- 14.50 & 90.10 +- 2.35 & 98.12 +- 0.71 & 98.44 +- 0.70 \\
fashion & 68.09 +- 18.97 & 98.40 +- 1.17 & 68.19 +- 17.40 & 95.19 +- 3.41 & 93.32 +- 1.54 & 87.15 +- 7.94 & 67.41 +- 19.07 & 98.27 +- 1.02 & 98.46 +- 1.12 \\
food & 77.09 +- 2.87 & 84.94 +- 2.43 & 74.00 +- 2.32 & 79.15 +- 2.80 & 78.54 +- 2.37 & 78.99 +- 2.99 & 76.71 +- 1.25 & 85.98 +- 2.08 & 85.09 +- 2.39 \\
garbage & 58.17 +- 30.04 & 78.03 +- 35.48 & 57.05 +- 26.93 & 61.52 +- 18.36 & 55.22 +- 20.04 & 55.37 +- 22.92 & 57.94 +- 30.02 & 79.11 +- 19.32 & 83.04 +- 26.66 \\
rock & 20.42 +- 8.11 & 50.03 +- 8.40 & 18.62 +- 6.03 & 42.51 +- 6.60 & 24.43 +- 7.29 & 15.89 +- 7.78 & 20.53 +- 8.00 & 39.72 +- 7.70 & 50.58 +- 7.83 \\
tissue & 72.44 +- 13.50 & 99.46 +- 0.62 & 72.58 +- 13.03 & 94.17 +- 3.16 & 91.63 +- 4.96 & 79.36 +- 20.33 & 76.15 +- 17.13 & 99.54 +- 0.59 & 99.47 +- 0.62 \\
yoga & 48.54 +- 10.87 & 58.68 +- 29.53 & 48.01 +- 10.65 & 61.87 +- 20.44 & 50.69 +- 16.47 & 50.05 +- 16.42 & 47.01 +- 12.07 & 63.11 +- 16.80 & 59.89 +- 26.66 \\
\bottomrule
\end{tabular}
}
\end{table*}

\subsection{AUPR-OUT}

\begin{table*}[htbp]
  \centering
  \caption{AUPR\_OUT (dino backbone). Per-dataset in-domain OOD results.}
  \label{tab:supp:extended:auprout:dino:in}
  \small
  \resizebox{\linewidth}{!}{\begin{tabular}{llllllllll}
\toprule
method_display & Energy-Based (EBO) & Mahalanobis (MDS) & MaxSoftmax (MSP) & Nearest Neigh. (kNN) & Neural Collapse (NCI) & ReAct & SCALE & Virtual-Logit (ViM) & whitening \\
ID Dataset &  &  &  &  &  &  &  &  &  \\
\midrule
colon & 83.45 +- 7.44 & 84.75 +- 5.08 & 80.42 +- 7.07 & 83.82 +- 7.70 & 85.67 +- 3.43 & 83.35 +- 7.26 & 86.22 +- 4.63 & 86.94 +- 4.64 & 83.92 +- 5.34 \\
eurosat & 76.90 +- 3.55 & 82.87 +- 3.69 & 74.09 +- 5.39 & 83.63 +- 2.80 & 76.69 +- 5.03 & 76.72 +- 4.30 & 75.04 +- 4.91 & 83.60 +- 3.48 & 82.77 +- 3.68 \\
fashion & 52.19 +- 13.80 & 63.18 +- 11.84 & 51.13 +- 11.15 & 65.68 +- 13.70 & 54.44 +- 16.85 & 55.13 +- 13.24 & 58.52 +- 15.83 & 63.90 +- 12.63 & 63.34 +- 11.99 \\
food & 49.47 +- 0.66 & 45.56 +- 0.71 & 47.67 +- 0.39 & 46.72 +- 0.90 & 49.41 +- 0.62 & 49.47 +- 0.56 & 48.78 +- 0.94 & 48.31 +- 0.63 & 46.47 +- 0.66 \\
garbage & 55.20 +- 12.67 & 56.39 +- 12.81 & 53.55 +- 12.37 & 48.08 +- 17.89 & 53.45 +- 15.04 & 55.40 +- 13.08 & 54.84 +- 13.28 & 58.19 +- 13.02 & 56.37 +- 12.82 \\
rock & 35.38 +- 7.71 & 29.21 +- 11.44 & 30.97 +- 6.94 & 33.37 +- 12.81 & 32.56 +- 8.06 & 36.48 +- 6.81 & 36.86 +- 6.03 & 31.60 +- 10.20 & 29.19 +- 11.38 \\
tissue & 45.55 +- 21.17 & 45.37 +- 20.94 & 43.79 +- 21.27 & 44.67 +- 21.64 & 39.54 +- 21.19 & 46.40 +- 21.95 & 45.42 +- 20.96 & 47.15 +- 20.73 & 44.97 +- 21.38 \\
yoga & 50.48 +- 10.47 & 43.63 +- 14.91 & 47.76 +- 10.04 & 45.11 +- 7.63 & 45.03 +- 13.08 & 43.96 +- 10.85 & 47.63 +- 10.75 & 47.40 +- 13.55 & 42.79 +- 13.06 \\
\bottomrule
\end{tabular}
}
\end{table*}

\begin{table*}[htbp]
  \centering
  \caption{AUPR\_OUT (dino backbone). Per-dataset out-of-domain OOD results.}
  \label{tab:supp:extended:auprout:dino:out}
  \small
  \resizebox{\linewidth}{!}{\begin{tabular}{llllllllll}
\toprule
method_display & Energy-Based (EBO) & Mahalanobis (MDS) & MaxSoftmax (MSP) & Nearest Neigh. (kNN) & Neural Collapse (NCI) & ReAct & SCALE & Virtual-Logit (ViM) & whitening \\
ID Dataset &  &  &  &  &  &  &  &  &  \\
\midrule
colon & 97.98 +- 2.45 & 99.98 +- 0.02 & 97.22 +- 2.65 & 99.62 +- 0.34 & 98.16 +- 2.79 & 98.14 +- 2.47 & 96.39 +- 3.08 & 99.98 +- 0.02 & 99.96 +- 0.03 \\
eurosat & 98.76 +- 0.28 & 99.76 +- 0.12 & 98.04 +- 0.47 & 98.62 +- 0.90 & 98.82 +- 1.04 & 98.94 +- 0.18 & 98.70 +- 0.34 & 99.80 +- 0.09 & 99.75 +- 0.12 \\
fashion & 91.39 +- 4.72 & 99.53 +- 0.31 & 87.62 +- 3.22 & 98.67 +- 0.63 & 94.67 +- 4.25 & 93.17 +- 4.34 & 93.25 +- 5.02 & 99.56 +- 0.35 & 99.48 +- 0.31 \\
food & 87.16 +- 1.40 & 88.12 +- 1.62 & 80.54 +- 0.86 & 85.08 +- 2.38 & 86.17 +- 1.47 & 87.50 +- 1.37 & 83.85 +- 2.59 & 87.53 +- 1.56 & 88.62 +- 1.55 \\
garbage & 98.80 +- 0.19 & 99.51 +- 0.34 & 98.64 +- 0.21 & 98.89 +- 0.39 & 98.77 +- 0.20 & 98.86 +- 0.11 & 98.79 +- 0.25 & 99.49 +- 0.30 & 99.51 +- 0.34 \\
rock & 97.99 +- 0.32 & 99.00 +- 0.23 & 98.20 +- 0.35 & 98.69 +- 0.31 & 98.52 +- 0.20 & 98.13 +- 0.35 & 98.13 +- 0.22 & 98.99 +- 0.17 & 99.00 +- 0.23 \\
tissue & 68.01 +- 9.29 & 99.97 +- 0.02 & 62.49 +- 3.07 & 99.72 +- 0.23 & 65.53 +- 9.97 & 70.60 +- 8.25 & 66.21 +- 9.23 & 99.94 +- 0.05 & 99.96 +- 0.02 \\
yoga & 98.11 +- 0.91 & 98.16 +- 1.98 & 97.93 +- 0.63 & 98.48 +- 0.65 & 98.40 +- 0.84 & 97.74 +- 1.12 & 98.17 +- 1.09 & 98.26 +- 1.17 & 98.30 +- 1.73 \\
\bottomrule
\end{tabular}
}
\end{table*}

\begin{table*}[htbp]
  \centering
  \caption{AUPR\_OUT (resnet backbone). Per-dataset in-domain OOD results.}
  \label{tab:supp:extended:auprout:resnet:in}
  \small
  \resizebox{\linewidth}{!}{\begin{tabular}{lllllllllll}
\toprule
method_display & Energy-Based (EBO) & Mahalanobis (MDS) & MaxSoftmax (MSP) & Nearest Neigh. (kNN) & Neural Collapse (NCI) & ReAct & SCALE & Virtual-Logit (ViM) & whitening & mds_teacher \\
ID Dataset &  &  &  &  &  &  &  &  &  &  \\
\midrule
colon & 83.21 +- 10.61 & 81.60 +- 6.93 & 82.98 +- 10.41 & 84.71 +- 8.98 & 65.76 +- 11.75 & 79.28 +- 8.87 & 82.63 +- 10.21 & 81.25 +- 6.09 & 83.35 +- 6.13 & 27.39 +- 5.06 \\
eurosat & 81.28 +- 5.94 & 81.16 +- 7.08 & 79.09 +- 5.15 & 79.18 +- 4.99 & 68.83 +- 6.15 & 77.73 +- 6.67 & 79.68 +- 6.07 & 80.20 +- 4.99 & 81.39 +- 5.32 & 54.21 +- 5.06 \\
fashion & 59.02 +- 18.86 & 61.58 +- 12.66 & 56.67 +- 14.74 & 61.15 +- 12.97 & 57.52 +- 12.03 & 51.54 +- 17.26 & 54.77 +- 18.19 & 61.11 +- 13.89 & 60.28 +- 13.71 & 60.87 +- 17.12 \\
food & 45.53 +- 1.29 & 45.40 +- 0.86 & 45.21 +- 1.13 & 45.33 +- 1.00 & 44.91 +- 0.88 & 44.73 +- 1.15 & 45.64 +- 1.17 & 44.18 +- 0.51 & 45.47 +- 0.87 & 40.29 +- 2.47 \\
garbage & 61.22 +- 12.17 & 63.15 +- 12.92 & 61.77 +- 10.97 & 51.54 +- 9.93 & 61.16 +- 11.73 & 56.24 +- 14.29 & 60.77 +- 11.54 & 66.22 +- 8.27 & 65.60 +- 9.02 & 62.71 +- 10.07 \\
rock & 32.94 +- 6.87 & 33.01 +- 5.17 & 35.02 +- 8.21 & 31.75 +- 7.73 & 34.61 +- 8.84 & 32.30 +- 6.52 & 33.14 +- 7.11 & 34.10 +- 5.53 & 34.73 +- 5.70 & 26.44 +- 8.79 \\
tissue & 48.82 +- 21.84 & 45.22 +- 21.64 & 46.65 +- 21.89 & 46.44 +- 21.66 & 47.28 +- 20.65 & 49.98 +- 20.82 & 49.24 +- 19.90 & 47.88 +- 20.98 & 60.08 +- 9.88 & 44.89 +- 21.06 \\
yoga & 51.97 +- 5.39 & 47.62 +- 5.54 & 51.23 +- 3.82 & 42.58 +- 7.24 & 49.99 +- 6.30 & 51.07 +- 5.88 & 50.94 +- 6.35 & 49.59 +- 5.46 & 47.42 +- 6.22 & 40.83 +- 6.29 \\
\bottomrule
\end{tabular}
}
\end{table*}

\begin{table*}[htbp]
  \centering
  \caption{AUPR\_OUT (resnet backbone). Per-dataset out-of-domain OOD results.}
  \label{tab:supp:extended:auprout:resnet:out}
  \small
  \resizebox{\linewidth}{!}{\begin{tabular}{lllllllllll}
\toprule
method_display & Energy-Based (EBO) & Mahalanobis (MDS) & MaxSoftmax (MSP) & Nearest Neigh. (kNN) & Neural Collapse (NCI) & ReAct & SCALE & Virtual-Logit (ViM) & whitening & mds_teacher \\
ID Dataset &  &  &  &  &  &  &  &  &  &  \\
\midrule
colon & 93.27 +- 1.47 & 97.92 +- 2.58 & 95.21 +- 1.01 & 98.19 +- 1.93 & 97.00 +- 1.65 & 85.13 +- 8.20 & 93.20 +- 1.33 & 98.32 +- 1.82 & 97.39 +- 3.04 & 99.84 +- 0.05 \\
eurosat & 97.22 +- 0.63 & 99.23 +- 0.28 & 96.93 +- 0.69 & 97.94 +- 0.60 & 95.78 +- 0.91 & 96.78 +- 0.90 & 97.00 +- 0.75 & 99.11 +- 0.32 & 99.11 +- 0.30 & 99.99 +- 0.00 \\
fashion & 97.65 +- 0.88 & 99.08 +- 0.41 & 95.39 +- 1.49 & 97.83 +- 0.81 & 94.44 +- 1.49 & 94.94 +- 3.01 & 96.89 +- 0.91 & 98.73 +- 0.62 & 99.03 +- 0.47 & 99.96 +- 0.02 \\
food & 77.68 +- 1.72 & 87.51 +- 1.59 & 76.47 +- 1.45 & 76.92 +- 1.36 & 76.20 +- 1.22 & 77.75 +- 1.63 & 77.74 +- 1.50 & 86.06 +- 0.84 & 87.70 +- 1.52 & 99.70 +- 0.06 \\
garbage & 99.34 +- 0.85 & 99.67 +- 0.59 & 99.28 +- 0.59 & 99.44 +- 0.29 & 99.48 +- 0.48 & 99.36 +- 0.71 & 99.30 +- 0.80 & 99.89 +- 0.12 & 99.79 +- 0.34 & 100.00 +- 0.00 \\
rock & 98.88 +- 0.30 & 99.53 +- 0.06 & 98.80 +- 0.28 & 99.06 +- 0.22 & 99.02 +- 0.18 & 98.78 +- 0.30 & 98.89 +- 0.29 & 99.43 +- 0.11 & 99.56 +- 0.05 & 99.89 +- 0.01 \\
tissue & 71.20 +- 5.17 & 99.28 +- 0.45 & 68.81 +- 5.31 & 87.39 +- 2.48 & 78.37 +- 2.59 & 76.49 +- 3.26 & 71.28 +- 4.23 & 99.18 +- 0.62 & 99.25 +- 0.35 & 100.00 +- 0.00 \\
yoga & 99.30 +- 0.25 & 98.91 +- 2.08 & 99.19 +- 0.23 & 99.06 +- 1.08 & 99.34 +- 0.28 & 99.30 +- 0.54 & 99.26 +- 0.40 & 99.14 +- 1.64 & 98.46 +- 2.55 & 99.96 +- 0.02 \\
\bottomrule
\end{tabular}
}
\end{table*}

\begin{table*}[htbp]
  \centering
  \caption{AUPR\_OUT (simclr backbone). Per-dataset in-domain OOD results.}
  \label{tab:supp:extended:auprout:simclr:in}
  \small
  \resizebox{\linewidth}{!}{\begin{tabular}{lll}
\toprule
method_display & Mahalanobis (MDS) & Nearest Neigh. (kNN) \\
ID Dataset &  &  \\
\midrule
colon & 65.70 +- 16.26 & 87.58 +- 5.54 \\
eurosat & 42.49 +- 15.86 & 83.71 +- 5.73 \\
fashion & 43.64 +- 9.77 & 66.92 +- 12.25 \\
food & 32.22 +- 1.48 & 51.28 +- 1.89 \\
garbage & 46.05 +- 3.62 & 46.46 +- 17.10 \\
rock & 25.46 +- 10.54 & 26.02 +- 10.84 \\
tissue & 42.85 +- 18.67 & 43.30 +- 22.64 \\
yoga & 39.22 +- 10.33 & 56.25 +- 8.20 \\
\bottomrule
\end{tabular}
}
\end{table*}

\begin{table*}[htbp]
  \centering
  \caption{AUPR\_OUT (simclr backbone). Per-dataset out-of-domain OOD results.}
  \label{tab:supp:extended:auprout:simclr:out}
  \small
  \resizebox{\linewidth}{!}{\begin{tabular}{lll}
\toprule
method_display & Mahalanobis (MDS) & Nearest Neigh. (kNN) \\
ID Dataset &  &  \\
\midrule
colon & 99.36 +- 0.31 & 95.82 +- 2.10 \\
eurosat & 99.05 +- 0.38 & 98.51 +- 0.22 \\
fashion & 99.58 +- 0.20 & 94.46 +- 1.24 \\
food & 72.20 +- 2.17 & 83.07 +- 3.80 \\
garbage & 99.92 +- 0.03 & 98.42 +- 0.63 \\
rock & 98.84 +- 0.11 & 98.53 +- 0.44 \\
tissue & 99.99 +- 0.01 & 78.62 +- 12.09 \\
yoga & 97.62 +- 0.57 & 99.05 +- 0.87 \\
\bottomrule
\end{tabular}
}
\end{table*}

\begin{table*}[htbp]
  \centering
  \caption{AUPR\_OUT (tra\_dino backbone). Per-dataset in-domain OOD results.}
  \label{tab:supp:extended:auprout:tra_dino:in}
  \small
  \resizebox{\linewidth}{!}{\begin{tabular}{llllllllll}
\toprule
method_display & Energy-Based (EBO) & Mahalanobis (MDS) & MaxSoftmax (MSP) & Nearest Neigh. (kNN) & Neural Collapse (NCI) & ReAct & SCALE & Virtual-Logit (ViM) & whitening \\
ID Dataset &  &  &  &  &  &  &  &  &  \\
\midrule
colon & 88.84 +- 5.38 & 88.70 +- 4.04 & 88.52 +- 3.99 & 89.56 +- 4.38 & 88.57 +- 4.27 & 88.59 +- 5.22 & 88.18 +- 4.30 & 91.12 +- 3.49 & 88.77 +- 4.01 \\
eurosat & 87.55 +- 2.61 & 80.49 +- 6.69 & 81.50 +- 5.52 & 83.62 +- 7.06 & 86.32 +- 2.97 & 87.53 +- 2.84 & 87.21 +- 3.01 & 86.44 +- 3.87 & 80.79 +- 6.61 \\
fashion & 63.52 +- 15.76 & 59.65 +- 13.58 & 58.60 +- 13.88 & 60.80 +- 15.51 & 64.41 +- 15.21 & 63.48 +- 15.54 & 61.87 +- 16.44 & 63.34 +- 14.36 & 59.76 +- 13.62 \\
food & 50.60 +- 0.83 & 46.19 +- 1.37 & 48.83 +- 0.71 & 48.09 +- 1.58 & 50.73 +- 0.65 & 50.60 +- 0.60 & 49.60 +- 1.52 & 49.03 +- 1.30 & 47.22 +- 1.39 \\
garbage & 60.22 +- 11.49 & 65.30 +- 7.54 & 56.20 +- 12.00 & 52.04 +- 9.63 & 54.93 +- 12.69 & 59.20 +- 12.13 & 60.67 +- 12.43 & 65.91 +- 9.43 & 65.29 +- 7.55 \\
rock & 33.84 +- 11.15 & 31.22 +- 9.79 & 30.06 +- 10.54 & 27.10 +- 7.52 & 30.27 +- 5.84 & 35.55 +- 10.94 & 34.48 +- 12.49 & 32.68 +- 9.37 & 30.90 +- 9.96 \\
tissue & 46.76 +- 21.39 & 46.23 +- 21.69 & 45.09 +- 21.71 & 44.78 +- 23.08 & 39.84 +- 21.19 & 46.56 +- 22.15 & 46.51 +- 22.86 & 48.16 +- 22.35 & 46.02 +- 21.96 \\
yoga & 48.43 +- 12.91 & 41.87 +- 6.61 & 46.18 +- 9.36 & 40.34 +- 9.11 & 42.59 +- 9.98 & 47.83 +- 12.38 & 46.63 +- 13.89 & 46.25 +- 9.99 & 41.90 +- 6.62 \\
\bottomrule
\end{tabular}
}
\end{table*}

\begin{table*}[htbp]
  \centering
  \caption{AUPR\_OUT (tra\_dino backbone). Per-dataset out-of-domain OOD results.}
  \label{tab:supp:extended:auprout:tra_dino:out}
  \small
  \resizebox{\linewidth}{!}{%
}
\end{table*}

\begin{table*}[htbp]
  \centering
  \caption{AUPR\_OUT (tra\_resnet backbone). Per-dataset in-domain OOD results.}
  \label{tab:supp:extended:auprout:tra_resnet:in}
  \small
  \resizebox{\linewidth}{!}{%
}
\end{table*}

\begin{table*}[htbp]
  \centering
  \caption{AUPR\_OUT (tra\_resnet backbone). Per-dataset out-of-domain OOD results.}
  \label{tab:supp:extended:auprout:tra_resnet:out}
  \small
  \resizebox{\linewidth}{!}{%
}
\end{table*}

\endgroup

\bibliographystyle{splncs04}
\bibliography{main}

\end{document}